\documentclass{article}

\usepackage[dblblindworkshop, final]{neurips_2026}
\workshoptitle{AI for Stochastic Dynamics}

\usepackage[utf8]{inputenc}
\usepackage[T1]{fontenc}
\usepackage{xurl}
\usepackage[hidelinks]{hyperref}
\usepackage{booktabs}
\usepackage{amsmath}
\usepackage{amssymb}
\usepackage{amsfonts}
\usepackage{amsthm}
\usepackage{mathtools}
\usepackage{nicefrac}
\usepackage{graphicx}
\usepackage{subcaption}
\usepackage{pdflscape}
\usepackage{algorithm}
\usepackage{algpseudocode}
\usepackage{microtype}
\usepackage{xcolor}

\graphicspath{{img/}{./img/}}

\newcommand{\E}{\mathbb{E}}

\newcommand{\CE}{\mathrm{CE}}

\newcommand{\Description}[1]{}

\newtheorem{lemma}{Lemma}
\newtheorem{proposition}{Proposition}
\newtheorem{corollary}{Corollary}

\title{StochLOB: Heavy-Tailed Score Matching for Robust Limit Order Book Trend Prediction}

\author{%
  \parbox{\dimexpr\textwidth-2\tabcolsep\relax}{\centering
    Arshia Abolghasemi\(^{1}\)\thanks{Equal contribution.\vspace{0.2 cm}} \quad
    Seyed Reza Tavakoli\(^{1}\)\footnotemark[1] \quad
    Parsa Naderi\(^{2}\) \quad
    Nazanin Mirzaei\(^{3}\) \\
    Behrad Jahani\(^{2}\) \quad
    Saba Yousefzadeh\(^{2}\) \quad
    Tahereh Javaheri\(^{4}\) \quad
    Mohammad Hosein Rohban\(^{2}\) \\
    Gholamali Aminian\(^{5}\) \quad
    Ali Habibnia\(^{1,6}\)\thanks{\vspace{-0.3 cm}Corresponding author: Ali Habibnia, \href{mailto:habibnia@vt.edu}{\texttt{habibnia@vt.edu}}.} \\[0.5em]
    \normalfont\small
    \(^{1}\)Dataism Laboratory for Quantitative Finance\\
    \(^{2}\)Sharif University of Technology \\
    \(^{3}\)Amirkabir University of Technology \\
    \(^{4}\)Boston University \\
    \(^{5}\)The Alan Turing Institute \\ \(^{6}\)Virginia Tech
  }%
}
\begin{document}

\maketitle

\begin{abstract}
Score-based diffusion methods usually perturb data with Brownian Gaussian
noise, corresponding to a light-tailed forward stochastic differential
equation (SDE). For market time series this SDE can be a poor abstraction:
limit order book (LOB) order flow is heavy-tailed, bursty, and occasionally
jump-like. We study \textsc{StochLOB}, a training-only score-matching
framework that changes the forward diffusion corruption law while leaving the
deployed predictor deterministic. Instead of Brownian-only Gaussian
corruption, the noising process uses two non-Gaussian SDE-inspired
alternatives: jump-\mbox{L\'evy} noise through a compound-Poisson variance
mixture, and clipped \(\alpha\)-stable noise through a positive-stable
Gaussian scale mixture. These corruptions inject finite jumps and
heavy-tailed residuals into training without requiring reverse-time diffusion
sampling at inference. The resulting score target has a useful
interpretation: its radial factor \(h_t(r)\) is the posterior expected
corruption precision, so large residuals are weighted according to the latent
corruption scale that most plausibly generated them. On chronologically
held-out Coinbase LOB data, the proposed variants remain competitive on
uncorrupted BTC and altcoin test horizons and improve several realized
tail-window stress diagnostics. An exploratory Feishu A-share transfer study
further tests whether the probability outputs remain usable under transaction
costs and trading-rule constraints.
Replication code and data-access information:
\url{https://github.com/DataismLab/StochLob}.
\end{abstract}

\section{Introduction}
\label{sec:intro}

Score-based diffusion methods are commonly built around a Brownian forward
stochastic differential equation (SDE), whose finite-time corruption is
Gaussian and light-tailed. This modeling choice is convenient, but it also
defines the residual geometry seen by a denoising score objective: large
coordinate movements become exponentially unlikely, and the learned restoring
force is tied to a Brownian quadratic penalty. For data generated by
stochastic systems with bursts, jumps, or persistent heavy tails, the
Brownian-only SDE can therefore be a poor training corruption law.

Limit order book (LOB) prediction makes this mismatch concrete. Order-flow
imbalance (OFI), depth changes, and short-horizon returns can remain quiet for
many bins and then move abruptly after an order sweep, a liquidity withdrawal,
or a short burst of directional trading. These large residuals are not rare
under the market process, but they are rare under Gaussian diffusion
corruption. This is difficult to reconcile with the heavy tails and jump-like
behavior observed in financial time series and LOB order flow
\citep{mandelbrot1963,cont2001stylized,cont2014ofi}.

The question we study is whether the forward diffusion SDE used for
score-based regularization should be modified so its noise law better matches
the stochastic input process. We introduce \textsc{StochLOB}, a training-only
framework for robust mid-price trend prediction. The deployed model remains a
standard classifier: the corruption sampler and score head are used only while
learning the shared representation, and test-time inference is a single
deterministic forward pass. This isolates the effect of changing the
stochastic diffusion corruption from the cost of reverse-time diffusion
sampling at deployment.

We consider two non-Gaussian SDE-inspired corruption families. The
jump-\mbox{L\'evy} variant adds a compound-Poisson jump component to a
Brownian background, implemented as a Gaussian scale mixture that exposes the
representation to mostly ordinary perturbations with occasional finite jumps.
The \(\alpha\)-stable variant replaces the fixed Brownian variance with a
clipped positive-stable scale mixture, producing broad heavy-tailed
perturbations while retaining numerical stability. Both variants use the same
classifier architecture and differ only in the stochastic training law and
auxiliary loss. The score target is also interpretable: the radial score
factor \(h_t(r)\) is a posterior precision under the latent corruption scale.
Thus the denoising objective does not apply the same restoring force to every
residual. Ordinary residuals receive high posterior precision, while extreme
residuals can be explained by larger latent corruption scales.

Our empirical analysis follows the same logic. First, distribution-fit and
tail diagnostics on Coinbase OFI show that the Gaussian assumption is a poor
description of the observed inputs. Second, on Coinbase BTC and four
additional crypto symbols, the proposed variants remain competitive on
standard held-out mid-price trend prediction and win selected horizons,
showing that the stochastic auxiliary objective does not require a large
standard-test accuracy tradeoff. Third, realized stress-window diagnostics
evaluate the regimes that
motivate the method: high-stress, high-variance, and jump-residual windows in
which Gaussian corruption is least plausible. Figure~\ref{fig:stress-main}
gives the headline stress diagnostic at \(k=100\). We also include an
exploratory Feishu A-share transfer study in which model probabilities are
converted into constrained portfolio signals, checking whether the forecasts
remain useful after transaction costs, lot-size rules, and T+1 selling
constraints. This transfer check is not meant to define a complete trading
system; it tests whether the learned probabilities survive a realistic
downstream decision rule.

\textbf{Contributions.} We make three contributions. First, we formulate
robust supervised representation learning as a mismatch between Brownian
diffusion-SDE corruption and heavy-tailed, jump-like stochastic inputs, using
LOB trend prediction as the concrete test bed. Second, we modify the
training-time diffusion corruption with jump-\mbox{L\'evy} and clipped
\(\alpha\)-stable noise while keeping the LOB predictor deterministic and free
of inference-time sampling. Third, we provide a posterior-precision
interpretation of the resulting score function \(h_t\), with proofs and
future-work mathematics in Appendix~\ref{app:h-score}.
Appendices~\ref{app:data}--\ref{app:feishu-portfolio} provide the data details,
score-table construction,
implementation and training algorithms, additional crypto results, and the
Feishu portfolio analysis.

\begin{figure}[t]
\centering
\includegraphics[width=\linewidth]{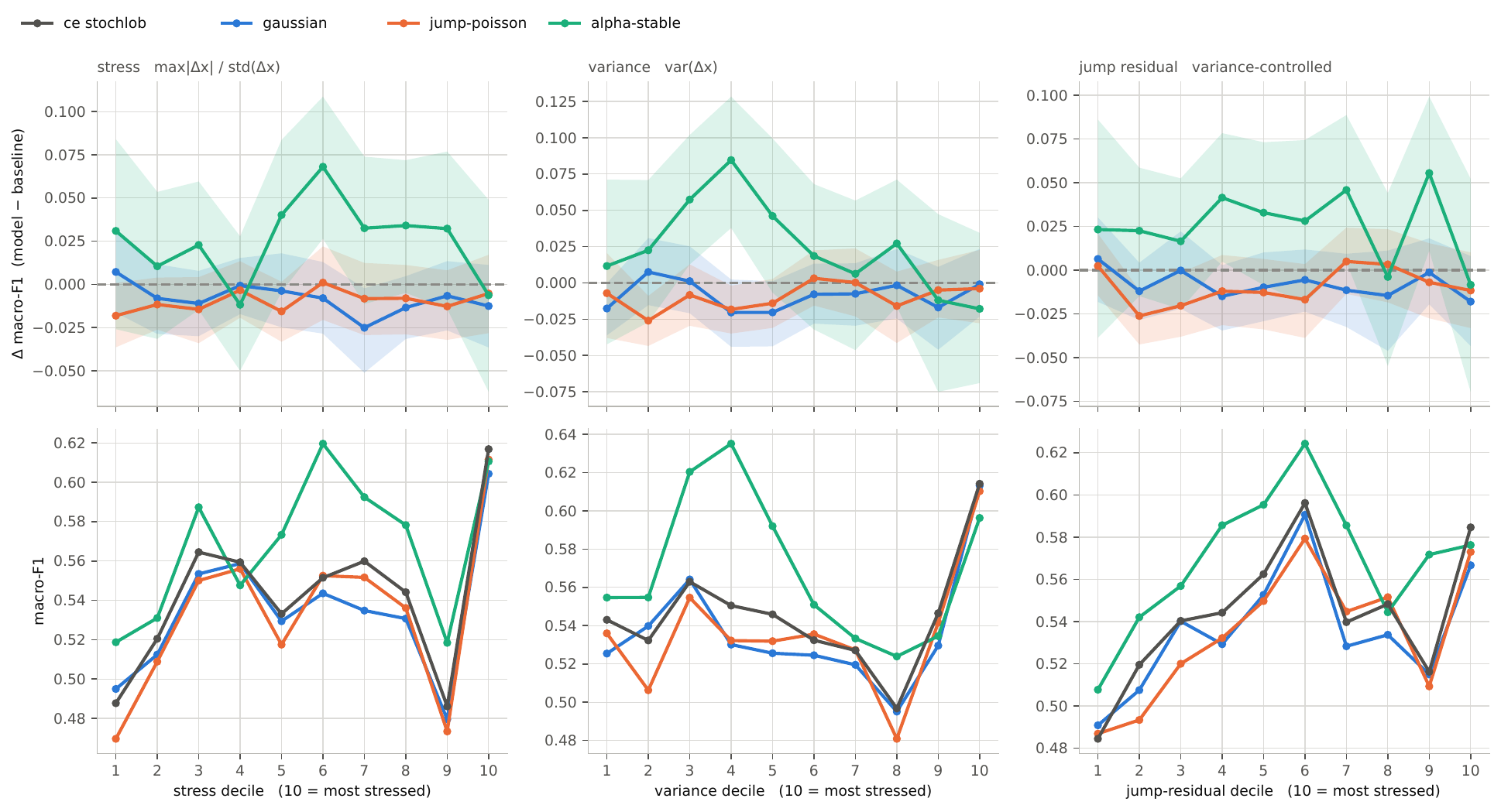}
\caption{Main stress result on chronologically held-out Coinbase BTC windows
at horizon \(k=100\) (\(31{,}842\) test windows). Columns sort windows by
realized stress, realized variance, and variance-controlled jump residual,
all computed from one-step increments of the level-averaged feature stream.
Decile 10 is the most turbulent. The top row reports macro-F1 by decile; the
bottom row reports paired differences against the same StochLOB trunk trained
with cross-entropy only. Shaded bands are \(95\%\) circular block-bootstrap
confidence intervals.}
\Description{Macro-F1 comparison for BTC at horizon k equals 100 across
high-stress, high-variance, and high-jump-residual test subsets.}
\label{fig:stress-main}
\end{figure}

\section{Related Work}
\label{sec:related}

\textbf{LOB trend prediction.}
LOB forecasting has moved from linear models toward architectures that exploit
both level structure and temporal order. DeepLOB combines convolutional,
Inception, and recurrent blocks for mid-price movement prediction
\citep{zhang2019deeplob}. TABL and BiN-TABL add bilinear mappings, temporal
attention, and input normalization for financial time series
\citep{tran2019tabl,tran2020bin}. Later work studies stronger attention and
benchmarking designs, including TLOB, LiT, LOBench, Axial-LOB, and
order-flow attention networks
\citep{berti2025tlob,lit2025,lobench2025,kisiel2022axiallob,bandealinaeini2023ofmatnet}.
Our experiments keep the standard three-class mid-price trend protocol and
compare against linear, recurrent, convolutional-attention, transformer, and
single-asset order-flow attention baselines trained on the same chronological
windows.

\textbf{Heavy tails and jumps.}
Heavy tails are a classical stylized fact of financial returns
\citep{mandelbrot1963,cont2001stylized}; LOB-specific studies show
non-Gaussian behavior in liquidity, impact, and order flow
\citep{bouchaud2002orderbook,mike2008liquidity,gould2013lob,cont2014ofi}.
Jump-diffusion models provide a finite-activity language for abrupt price
moves \citep{merton1976jump}, while stable processes model persistent
heavy-tailed fluctuations \citep{samorodnitsky1994stable}. We use these
objects as training-time corruption laws, not as complete market simulators.

\textbf{Diffusion as representation learning.}
Score-based and denoising diffusion models are typically introduced as
generative models \citep{ho2020ddpm,song2021sde}, but the denoising task can
also shape supervised representations. Joint diffusion training uses shared
networks for denoising and prediction \citep{deja2023jointdiffusion}, and
related work studies diffusion features as representation learners
\citep{yang2023repfusion}. Non-Gaussian diffusion replaces Brownian noise with
jump or heavy-tailed processes
\citep{yoon2023levy,shariatian2025dlpm,pandey2025heavytailed}. LOB diffusion
work models event streams, counterfactual books, or volume snapshots
\citep{zheng2024lobdif,wang2026difflob,wang2025diffvolume}; \textsc{StochLOB}
instead uses non-Gaussian corruption as a training-time auxiliary regularizer
for deterministic LOB classification.

\section{Background}
\label{sec:background}

Limit order book (LOB) data record the outstanding limit orders waiting at an
exchange. At any time, the book has a bid side, where traders are willing to
buy, and an ask side, where traders are willing to sell. Each side is ordered
by price into levels: the first bid and ask levels are the immediately
executable quotes, while deeper levels describe nearby liquidity that would
be reached after the top quotes are consumed. A raw book snapshot is therefore
a structured row rather than a flat tabular observation: for each displayed
level it contains a bid price, bid volume, ask price, and ask volume. With
twenty displayed levels, one snapshot contains eighty raw book fields before
trade-flow and derived features are added.

Across time, consecutive snapshots show how queues build, vanish, or move
across levels. These changes are the main signal in LOB prediction because
they reveal buying and selling pressure before it appears as a mid-price move.

Following the multi-level OFI representation of \citet{kolm2023deepofi},
which builds on the order-flow imbalance construction of \citet{cont2014ofi},
we use OFI as the primary LOB representation. OFI summarizes whether
liquidity is added, cancelled, or consumed on the bid and ask sides before
the mid-price moves. New bid depth or removed ask depth pushes signed flow
positive; removed bid depth or new ask depth pushes it negative. Aggregating
these signed changes converts irregular book updates into a regular
time-series signal for neural sequence models. Appendix~\ref{app:lob-representation}
gives the full LOB notation, Coinbase feature construction, and exact OFI
aggregation formula in Eq.~\eqref{eq:ofi-window}. Appendix~\ref{app:tail-evidence}
reports the train-only heavy-tail diagnostics, including
Figures~\ref{fig:diagnostic-fits-main} and \ref{fig:crypto-tail}.

The representation-learning question is whether the auxiliary corruption task
should match the market residuals seen by OFI windows. Many bins contain
small queue adjustments, but active periods can contain abrupt order sweeps,
liquidity withdrawals, or bursts of one-sided trading. A pooled Hill/Pareto
tail-index estimate on Coinbase OFI magnitudes gives
\(\hat{\alpha}=1.62\), well below Gaussian light-tail behavior
(Appendix~\ref{app:tail-evidence}). A Gaussian denoising task emphasizes a
fixed quadratic geometry around the clean input; the
non-Gaussian score tasks below instead expose the shared classifier trunk to
larger residuals during training. The reverse diffusion chain is never run at
inference, so the method is stochastic representation learning rather than a
generative market simulator.

\section{Methodology}
\label{sec:methodology}

\textsc{StochLOB} is designed to change the training signal while leaving the
deployed predictor unchanged. Each minibatch is seen in two views. The clean
window \(x_0\) is passed through the classifier path and optimized with the
ordinary three-class trend loss. A second, stochastically corrupted view
\(x_t\) is passed through a score head attached to the same temporal trunk.
Only the shared representation is common between these objectives; at test
time the noisy view, corruption sampler, score table, and score head are all
disabled.

\subsection{Problem and Labels}
\label{sec:research-problem}

Given a normalized LOB feature window \(x_0\in\mathbb{R}^{T\times F}\), the
classifier returns probabilities over
\(\{\mathrm{down},\mathrm{stationary},\mathrm{up}\}\). For Coinbase,
\(T=60\) and \(F=31\): each window covers 10 minutes sampled every 10 seconds.
Following the DeepLOB smoothed-mid protocol \citep{zhang2019deeplob}, the
return target at horizon \(k\) is
\begin{equation}
  r_k(\tau)=
\frac{
    \frac{1}{k}\sum_{i=1}^{k} p_{\tau+i}
    -
    \frac{1}{k}\sum_{i=1}^{k} p_{\tau-i}
}{
    \frac{1}{k}\sum_{i=1}^{k} p_{\tau-i}
},
  \label{eq:label-return}
\end{equation}
and labels are assigned by train-calibrated thresholds
\begin{equation}
  y_\tau =
  \begin{cases}
    \mathrm{up}, & r_k(\tau)>\alpha_{\mathrm{lab}},\\
    \mathrm{down}, & r_k(\tau)<-\alpha_{\mathrm{lab}},\\
    \mathrm{stationary}, & |r_k(\tau)|\le \alpha_{\mathrm{lab}}.
  \end{cases}
  \label{eq:trend-label}
\end{equation}
The standard supervised loss is
\(\mathcal{L}_{\mathrm{cls}}=\CE(f_\theta(x_0),y)\).

The chronological construction is important. For a window centered at
\(\tau\), the input uses only observations up to \(\tau\), while the label
uses the smoothed future mid-price in Eq.~\eqref{eq:label-return}. Windows are
therefore assigned to train, validation, or test only when the full look-back
context and the full future label horizon remain inside the same split. This
avoids leakage from overlapping windows near split boundaries.

\subsection{Generalized Scale-Mixture Score Matching}
\label{sec:non-gaussian-objective}

\textsc{StochLOB} corrupts an unperturbed input through a Gaussian scale
mixture:
\begin{equation}
  x_t = a_t x_0 + u,\qquad
  u=\sqrt{W_t}\xi,\qquad
  \xi\sim\mathcal{N}(0,I_d).
  \label{eq:gsm-main}
\end{equation}
Here \(a_t\) follows a standard variance-preserving or cosine schedule, while
the variance \(W_t\) is random. Gaussian diffusion is the special case
\(W_t=\sigma_t^2\). Non-Gaussian corruption changes only the scale law
\(p_t(W)\), so the same classifier can be trained under light-tailed,
jump-like, or heavy-tailed perturbations.

We choose Gaussian scale mixtures because they keep the score target
tractable while allowing non-Gaussian residuals. Conditional on \(W_t=w\),
the corruption is Gaussian and isotropic. Marginally, however, the variance
is random, so the residual distribution can place much more probability on
large moves. This gives a controlled way to expose the classifier
representation to market-shaped perturbations without changing the input
schema or the inference-time architecture.

The additive-noise density is
\begin{equation}
  q_t(u)=\E_{W_t}\left[
  (2\pi W_t)^{-d/2}\exp\left(-\frac{\lVert u\rVert^2}{2W_t}\right)
  \right].
  \label{eq:gsm-noise-density}
\end{equation}
Its score is radial:
\begin{equation}
  \nabla_u\log q_t(u)=-u\,h_t(\lVert u\rVert),
  \label{eq:gsm-score-main}
\end{equation}
where
\begin{equation}
  h_t(r)=
  \frac{\E_W[W^{-1}W^{-d/2}\exp(-r^2/2W)]}
       {\E_W[W^{-d/2}\exp(-r^2/2W)]}.
  \label{eq:radial-score}
\end{equation}
The function \(h_t(r)\) can be interpreted as the posterior expected corruption precision,
$\mathbb{E}\!\left[W_t^{-1}\mid \lVert u\rVert=r\right]$,
as shown in more detail in Appendix~\ref{app:h-score}
(Posterior-Precision Interpretation of \(h_t\)). Following
\citet{baule2025}, \textsc{StochLOB} tabulates this one-dimensional function
by Monte Carlo over residual radii and diffusion steps, then interpolates it
during training. Appendix~\ref{app:gsm-monte-carlo} establishes that this
tabulation procedure is a self-normalized importance-sampling estimator
\citep{geweke1989bayesian} of a Tweedie-type posterior expectation
\citep{efron2011tweedie}. For each fixed \(r>0\), the estimator is strongly
consistent, satisfies an explicit non-asymptotic error bound, and is
asymptotically normal with a closed-form variance under the Gaussian,
jump-Poisson, and clipped alpha-stable corruption laws used in this paper.
This posterior-score view is a Gaussian scale-mixture instance of classical
empirical-Bayes score identities \citep{andrews1974scale,efron2011tweedie}.

The score loss is
\begin{equation}
  \mathcal{L}_{\mathrm{score}}
  =
  \E_{t,W_t,\xi}\left[
    w_t
    \left\|
      s_\theta(x_t,t)+u\,h_t(\lVert u\rVert)
    \right\|^2_{\mathrm{mean}}
  \right].
  \label{eq:score-loss-main}
\end{equation}
The scalar \(w_t\) is a deterministic per-step denoising-score-matching
weight: \(w_t=\E[W_t]\) for finite-variance Gaussian and jump corruptions, and
\(w_t=1\) for the reported clipped \(\alpha\)-stable runs. The classifier and
score head share the temporal trunk, but only the classifier path is used at
test time. The complete training objective is
\[
  \mathcal{L}
  =
  \mathcal{L}_{\mathrm{cls}}
  +
  \lambda_{\mathrm{diff}}\mathcal{L}_{\mathrm{score}}.
\]

\subsection{Two Market-Shaped Corruptions}
\label{sec:market-corruptions}

\textbf{\textsc{StochLOB} (\mbox{L\'evy}).}
The jump variant uses a Brownian background plus a compound-Poisson gamma
variance:
\begin{equation}
  W_t=\sigma_t^2+\sum_{j=1}^{N_t}S_j,\qquad
  N_t\sim\mathrm{Poisson}(\Lambda_t),\quad
  S_j\sim\mathrm{Gamma}(\kappa,\theta).
  \label{eq:jgl-mixture}
\end{equation}
This exposes the representation to mostly ordinary perturbations with
occasional large shocks, matching the intuition of order sweeps or liquidity
withdrawals.

\textbf{\textsc{StochLOB} (\(\alpha\)-stable).}
The stable variant uses
\begin{equation}
  W_t=\sigma_t^2 A,\qquad
  A\sim\mathcal{S}^{+}_{\alpha/2},\qquad
  \E[\exp(-sA)]=\exp(-s^{\alpha/2}).
  \label{eq:asl-mixture}
\end{equation}
This creates polynomial-tail perturbations in the ideal case. The
implementation clips \(A\) at a high quantile for numerical stability, giving
a truncated-stable corruption law that remains much heavier-tailed than a
Gaussian. Appendix~\ref{app:data} (Data Details, Alpha Selection, and Tail
Diagnostics), especially Appendices~\ref{app:label-alpha}
and \ref{app:stable-alpha}, explains how \(\alpha\) and
\(\alpha_{\mathrm{lab}}\) are selected without test leakage.

\subsection{Architecture}
\label{sec:architecture-algorithm}

All proposed variants use the same classifier family: optional BiN
normalization, a two-layer bidirectional GRU, one temporal self-attention
block with timestep conditioning, attention pooling, and a three-class
classifier head. The score head is a training-only branch attached to the
shared representation. Figure~\ref{fig:model-architecture} shows the complete
training workflow and the inference path without stochastic corruption.
Using a shared architecture isolates the effect of the corruption law: changes
in test behavior are attributable to the auxiliary stochastic objective rather
than to a larger inference model.
Training, inference, and score-table pseudocode are given in
Appendix~\ref{app:gsm} (GSM Score Table) and Appendix~\ref{app:algorithms}
(Training and Inference Pseudocode).

\clearpage
\begin{landscape}
\begin{figure}[p]
\centering
\includegraphics[width=0.98\linewidth,keepaspectratio]{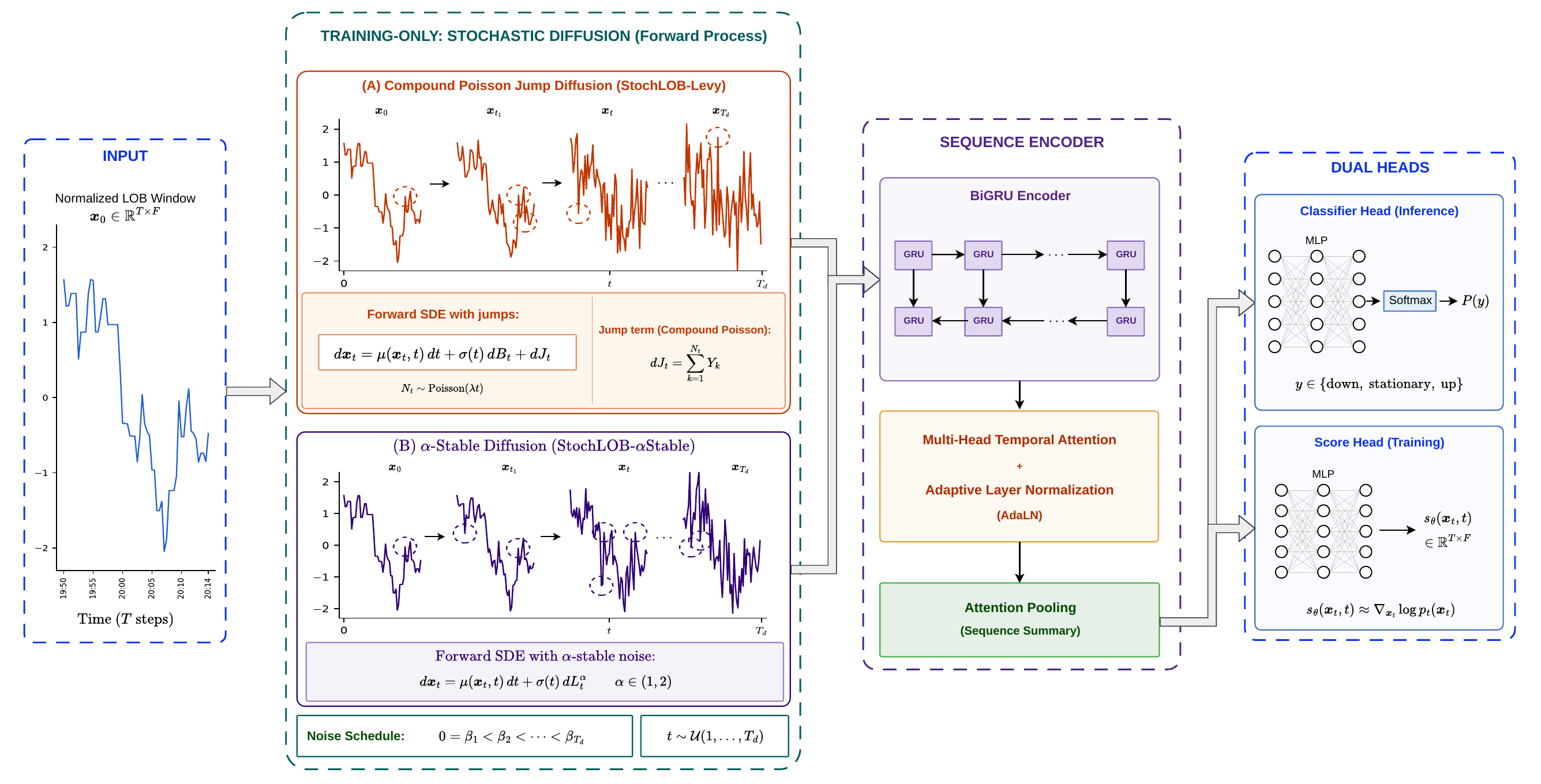}
\caption{StochLOB workflow. The classifier pools the unperturbed sequence
representation for three trend logits. The training-only score decoder
receives the noisy \(T\times F\) window, timestep embedding, and per-timestep
trunk context, and predicts a \(T\times F\) score field. At inference, the
corruption sampler and score decoder are disabled.}
\Description{Architecture diagram of the proposed StochLOB model. A normalized
LOB time-series is encoded by a BiGRU feature extractor, refined by temporal
attention, and summarized for classification. A separate score decoder
predicts a full time-feature score field during training only.}
\label{fig:model-architecture}
\end{figure}
\end{landscape}
\clearpage

\section{Datasets}
\label{sec:datasets}

We use point-in-time market data from crypto and equity order books.
Appendix~\ref{app:data} gives the exact timestamps, feature construction,
normalization, split counts, label calibration, and tail diagnostics.

\textbf{Coinbase crypto.}
The main benchmark is Coinbase order-book and trade data for BTC, ETH, BNB,
XRP, and USDT. The main text focuses on BTC as the deepest symbol; ETH, BNB,
XRP, and USDT results are in Appendix~\ref{app:results}.

\textbf{Feishu A-share universe.}
Feishu is a broader Chinese A-share transfer setting. Each asset-day combines
intraday OFI summaries with daily OHLCV and technical variables. The transfer
study evaluates the resulting directional probabilities through a common
constrained portfolio rule, described in Appendix~\ref{app:feishu-portfolio}.
\section{Experiments}
\label{sec:experiments}

\subsection{Experiment Setup}
\label{sec:experiment-setup}

All Coinbase experiments use chronological 70/15/15 train, validation, and
test splits, so every model is evaluated on later market periods than those
used for training. Checkpoints are selected by validation macro-F1 and then
frozen for all reported test metrics. The main tables use OFI windows with
\(T=60\), \(F=31\), and horizons \(k\in\{10,20,50,100\}\). Optimizer,
corruption, seed-sweep, and Feishu transfer details are given in
Appendix~\ref{app:models}.

\subsection{Mid-Price Trend Prediction Results}
\label{sec:trend-results}

Table~\ref{tab:btc-trend-main} reports standard held-out BTC mid-price trend
prediction results with seed-sweep variability. StochLOB (jump-Poisson) has the
highest mean macro-F1 at \(k=10\), StochLOB (alpha-stable) at \(k=20\),
and StochLOB variants are competitive at the remaining horizons. The
standard-test ranking is mixed, so the proposed variants are not presented as
across-the-board leaderboard winners. The primary claim is that non-Gaussian
score corruption remains competitive on BTC and improves stability in realized
tail-window regimes.

\begin{table*}[t]
\centering
\small
\setlength{\tabcolsep}{6pt}
\renewcommand{\arraystretch}{1.1}
\begin{tabular}{@{}lcccc@{}}
\toprule
& \multicolumn{2}{c}{$k=10$} & \multicolumn{2}{c}{$k=20$} \\
\cmidrule(lr){2-3}
\cmidrule(lr){4-5}
Model & Acc & F1 & Acc & F1 \\
\midrule
Logistic regression
& 0.6607$\pm$0.0007 & 0.6591$\pm$0.0007
& 0.6061$\pm$0.0049 & 0.6089$\pm$0.0048 \\

Linear VAR
& 0.6735$\pm$0.0009 & 0.6710$\pm$0.0008
& 0.6205$\pm$0.0017 & 0.6231$\pm$0.0015 \\

DeepLOB (CNN-LSTM)
& 0.6338$\pm$0.0167 & 0.6284$\pm$0.0142
& 0.5591$\pm$0.0128 & 0.5535$\pm$0.0079 \\

BiN-CTABL
& 0.6679$\pm$0.0045 & 0.6659$\pm$0.0041
& 0.6097$\pm$0.0027 & 0.6110$\pm$0.0024 \\

TLOB (Transformer)
& 0.6783$\pm$0.0068 & 0.6754$\pm$0.0054
& 0.6127$\pm$0.0076 & 0.6142$\pm$0.0068 \\

OFSATNet
& 0.6834$\pm$0.0027 & 0.6811$\pm$0.0026
& 0.6190$\pm$0.0109 & 0.6205$\pm$0.0099 \\
\midrule

StochLOB (jump-Poisson)
& \textbf{0.6891$\pm$0.0036} & \textbf{0.6866$\pm$0.0031}
& \underline{0.6212$\pm$0.0125} & \underline{0.6230$\pm$0.0117} \\

StochLOB (alpha-stable)
& \underline{0.6885$\pm$0.0069} & \underline{0.6848$\pm$0.0061}
& \textbf{0.6274$\pm$0.0036} & \textbf{0.6275$\pm$0.0028} \\

\bottomrule
\end{tabular}

\medskip

\begin{tabular}{@{}lcccc@{}}
\toprule
& \multicolumn{2}{c}{$k=50$} & \multicolumn{2}{c}{$k=100$} \\
\cmidrule(lr){2-3}
\cmidrule(lr){4-5}
Model & Acc & F1 & Acc & F1 \\
\midrule
Logistic regression
& 0.5802$\pm$0.0040 & 0.5853$\pm$0.0039
& 0.5493$\pm$0.0019 & 0.5538$\pm$0.0019 \\

Linear VAR
& 0.5948$\pm$0.0020 & 0.5998$\pm$0.0020
& \textbf{0.5632$\pm$0.0016} & \textbf{0.5689$\pm$0.0017} \\

DeepLOB (CNN-LSTM)
& 0.5311$\pm$0.0140 & 0.5233$\pm$0.0176
& 0.5138$\pm$0.0284 & 0.5133$\pm$0.0300 \\

BiN-CTABL
& 0.5736$\pm$0.0101 & 0.5764$\pm$0.0096
& 0.5411$\pm$0.0138 & 0.5445$\pm$0.0139 \\

TLOB (Transformer)
& 0.5856$\pm$0.0201 & 0.5870$\pm$0.0185
& 0.5574$\pm$0.0127 & 0.5604$\pm$0.0117 \\

OFSATNet
& 0.5828$\pm$0.0126 & 0.5861$\pm$0.0102
& 0.5418$\pm$0.0136 & 0.5438$\pm$0.0149 \\
\midrule

StochLOB (jump-Poisson)
& \textbf{0.6009$\pm$0.0152} & \textbf{0.6044$\pm$0.0140}
& 0.5536$\pm$0.0200 & 0.5570$\pm$0.0212 \\

StochLOB (alpha-stable)
& \underline{0.5929$\pm$0.0069} & \underline{0.5973$\pm$0.0066}
& \underline{0.5619$\pm$0.0137} & \underline{0.5663$\pm$0.0134} \\
\bottomrule
\end{tabular}
\caption{BTC OFI trend prediction on chronologically held-out windows.
Accuracy (Acc) and macro-F1 are reported as mean\(\pm\)std over five seeds
(\(13,21,42,84,168\)). Bold marks the best mean result in each horizon/metric
column; underlining marks the second best.}
\label{tab:btc-trend-main}
\end{table*}

\subsection{Stress Stability}
\label{sec:stress-stability}

Standard held-out trend prediction does not fully test behavior under
difficult market states. Figure~\ref{fig:stress-main} shows the BTC stress
comparison, and Table~\ref{tab:stress-extreme-metrics} reports a compact
summary of extreme-window results across stress, high-variance, and
jump-residual regimes. For each normalized OFI window, we average features
across levels and compute one-step increments \(\Delta x\). The stress score is
\(\max |\Delta x| / \mathrm{std}(\Delta x)\), the high-variance score is
\(\mathrm{var}(\Delta x)\), and the jump-residual score is the residual from
regressing \(\log \max |\Delta x|\) on \(\log \mathrm{var}(\Delta x)\) within
the same horizon. Extreme windows are the top decile of the corresponding
score. The full pre-specified regime-by-horizon grid is in Appendix
Table~\ref{tab:stress-extreme-full-app}, with implementation details in
Appendix~\ref{app:models}. These comparisons support the paper's empirical
claim: StochLOB variants are not uniformly best on standard held-out data, but
heavy-tailed score corruption gives measurable stability gains in several
tail-region test windows.

\begin{table*}[t]
\centering
\footnotesize
\setlength{\tabcolsep}{3pt}
\renewcommand{\arraystretch}{1.12}
\begin{tabular}{@{}llcccccc@{}}
\toprule
Regime & \shortstack{Best\\StochLOB} & \(k\) & \shortstack{StochLOB\\F1} &
\shortstack{Gaussian\\F1} & \shortstack{Best\\supervised} &
\shortstack{Gain vs.\\Gaussian} & \shortstack{Gain vs.\\supervised} \\
\midrule
Stress & alpha-stable & 10 & \textbf{0.712} & 0.703 & TLOB 0.708 & +0.008 & +0.004 \\
Stress & alpha-stable & 100 & \textbf{0.647} & 0.604 & LogReg 0.614 & +0.043 & +0.034 \\
High variance & jump-Poisson & 20 & \textbf{0.686} & 0.680 & LinVAR 0.651 & +0.006 & +0.034 \\
High variance & alpha-stable & 50 & \textbf{0.658} & 0.603 & LinVAR 0.645 & +0.055 & +0.014 \\
High variance & alpha-stable & 100 & \textbf{0.659} & 0.613 & LinVAR 0.628 & +0.046 & +0.030 \\
Jump residual & alpha-stable & 100 & \textbf{0.637} & 0.567 & LinVAR 0.597 & +0.070 & +0.039 \\
\bottomrule
\end{tabular}
\caption{Compact summary of extreme-window macro-F1 comparisons on held-out
BTC regimes. Each row reports the best heavy-tailed StochLOB variant in that
regime and horizon, the matched Gaussian StochLOB baseline, and the strongest
supervised comparator in the same row. \textbf{Best F1 in each row is shown in
bold.} Gains are absolute macro-F1 differences; Appendix
Table~\ref{tab:stress-extreme-full-app} reports the full regime-by-horizon grid.}
\label{tab:stress-extreme-metrics}
\end{table*}

\subsection{Feishu Portfolio Transfer}

Table~\ref{tab:feishu-portfolio-main} and
Figure~\ref{fig:feishu-portfolio-main} report the Feishu out-of-sample
portfolio transfer check. The classifier is trained for mid-price direction;
the fixed downstream portfolio rule tests whether its probabilities remain
usable after frictions and trading constraints. Details are in
Appendices~\ref{app:feishu-model-targets}, \ref{app:feishu-portfolio-setup},
and \ref{app:feishu-portfolio-results}. Among the active models shown,
StochLOB (jump-Poisson) reaches 38.78\% CAGR and Sharpe 1.04, with the lowest
maximum drawdown and volatility.

\begin{table*}[t]
\centering
\scriptsize
\setlength{\tabcolsep}{3pt}
\renewcommand{\arraystretch}{1.08}
\resizebox{\textwidth}{!}{%
\begin{tabular}{@{}lrrrrrrrr@{}}
\toprule
Model & CAGR(\%) & Sharpe & Max DD(\%) & Return(\%) & Vol.(\%) & Calmar & Win(\%) & Avg. Hold.\\
\midrule
LogReg
& 28.03
& 0.69
& $-$28.41
& 28.03
& 46.28
& 0.99
& 47.72
& 167.3\\

LinVAR
& $-$38.66
& $-$0.71
& $-$49.75
& $-$38.66
& 52.35
& $-$0.78
& 43.98
& 11.3\\

DeepLOB
& 31.72
& 0.68
& $-$36.49
& 31.72
& 82.61
& 0.87
& 48.55
& 50.1\\

TLOB
& 17.56
& 0.55
& $-$29.04
& 17.56
& 46.28
& 0.60
& 53.11
& 961.5\\
\midrule
StochLOB (jump-Poisson)
& \textbf{38.78}
& \textbf{1.04}
& \textbf{$-$23.20}
& \textbf{38.78}
& \textbf{36.64}
& \textbf{1.67}
& \textbf{54.36}
& \textbf{1861.0}\\
\bottomrule
\end{tabular}}
\caption{Feishu out-of-sample constrained portfolio metrics. Bold denotes the
best result in each column across the shown active models. Other evaluated
models are omitted because they mostly predict the flat class and therefore
generate little or no trading activity.}
\label{tab:feishu-portfolio-main}
\end{table*}

\begin{figure}[t]
\centering
\includegraphics[height=0.24\textheight,keepaspectratio]{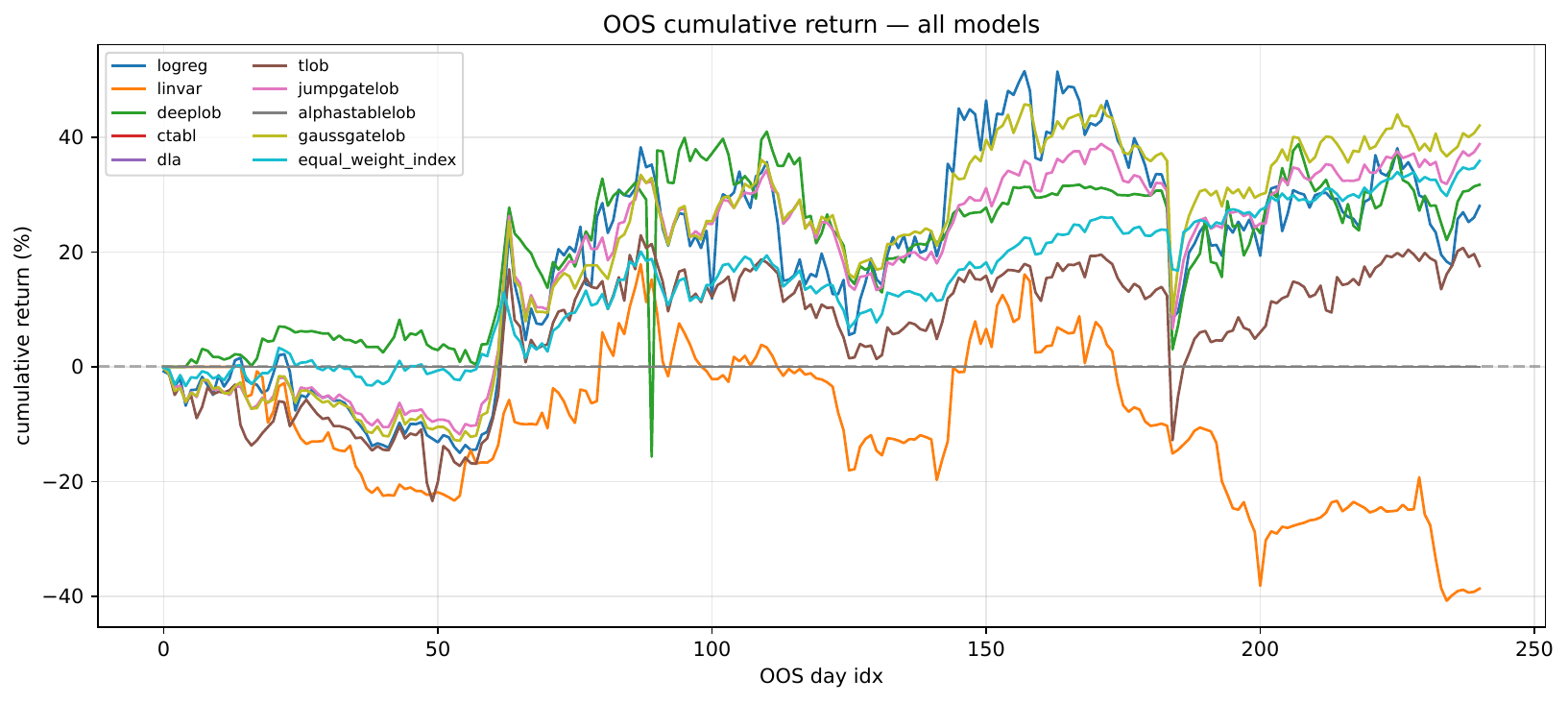}
\caption{Exploratory Feishu portfolio-management transfer result. Cumulative
returns are shown for the constrained portfolio rule using StochLOB variants,
supervised baselines, and the passive equal-weight reference.}
\Description{Line plot of Feishu out-of-sample cumulative returns for model
signals and the equal-weight reference.}
\label{fig:feishu-portfolio-main}
\end{figure}

\section{Discussion}
\label{sec:discussion}

The central message of \textsc{StochLOB} is not that one architecture should
dominate every standard benchmark cell. It is that the auxiliary corruption
used for representation learning should match the regime in which robustness is
needed. Crypto OFI inputs are heavy-tailed, ordinary Gaussian corruption is a
limited proxy for market stress, and non-Gaussian score matching gives the
shared representation a geometry that is more appropriate for large-residual
states. The posterior-precision view of \(h_t\) makes this mechanism explicit:
as the residual magnitude changes, the score adapts to the corruption scale
that is plausible under the chosen noise law.

There are limits. The compound-Poisson law is still isotropic after scaling;
real book shocks are signed, level-correlated, and state dependent. The
alpha-stable law is clipped for numerical stability, so the implementation is
a truncated-stable training distribution rather than a pure infinite-variance
process. Feishu also shows that robust probabilities are not automatically
portfolio-ready; the alpha-stable model may need threshold or utility-aware
calibration before its conservative signal becomes economically useful.

Future work should turn this into a sharper coverage criterion. The present
corruptions improve radial coverage by making large residual magnitudes
plausible during training, but they do not guarantee angular coverage: real
book stress is signed, cross-level, and concentrated along a small number of
market directions. Appendix~\ref{app:h-score}, especially
Appendix~\ref{app:future-work-math}, formalizes this with a second-order
change-of-measure argument
in which robustness transfer depends on both radial tail compatibility and an
angular R\'enyi term that can dominate in high-dimensional LOB windows. This
motivates four next steps: estimate state-dependent directional stress spectra
and replace scalar \(h_t(r)\) with
matrix-valued posterior precision; replay signed empirical jumps from real
order-book residuals; separate the score loss and temporal readout with
multi-seed ablations; and select portfolio thresholds on validation data under
transaction costs and risk constraints.

\section{Conclusion}
\label{sec:conclusion}

This work reframes LOB diffusion regularization around the noise geometry
observed in market data. \textsc{StochLOB} keeps inference as a deterministic
classifier, but trains the shared representation with jump-like or
heavy-tailed score-matching corruption. The resulting models remain
competitive on standard Coinbase trend prediction, improve several realized
stress-window diagnostics, and produce probabilities that can be tested under
a constrained Feishu portfolio rule. Appendix~\ref{app:h-score}
(Posterior-Precision Interpretation of \(h_t\)) proves that \(h_t\) is a
posterior expected corruption precision; Appendices~\ref{app:algorithms}
(Training and Inference Pseudocode) and~\ref{app:models} (Experimental Setup
and Implementation Details) provide the operational details.

\noindent\begin{minipage}{\linewidth}
\section*{Declarations}

\paragraph{Funding and computational resources.}
This study received no project-specific research funding.
No computing resources from Virginia Tech or the authors' other affiliated
universities were used in this study.

\paragraph{Code and data availability.}
The analyses used publicly available market data. Data coverage,
preprocessing, and chronological split construction are described in
Section~\ref{sec:datasets} and Appendix~\ref{app:data}. Replication code
and data-access information are provided in the project repository at
\url{https://github.com/DataismLab/StochLob}.
\end{minipage}\par

\bibliographystyle{plainnat}
\bibliography{references}

\clearpage
\appendix

\section{Data Details, Alpha Selection, and Tail Diagnostics}
\label{app:data}

\subsection{Coinbase Crypto Protocol and LOB Representation}
\label{app:lob-representation}

The Coinbase dataset covers BTC, ETH, BNB, XRP, and USDT. The
raw data were crawled from Coinbase as paired order-book and trade CSV files
and then normalized into the shared parquet schema used by the experiments.
The feed is sampled every 10 seconds from 2026-06-18 23:42:10 UTC to
approximately 2026-07-13 23:10:40 UTC, giving 26 calendar dates and roughly
213k snapshots per symbol. Each input window has \(T=60\) steps, so it covers
10 minutes of recent book and trade-flow state. Prediction horizons
\(k=\{10,20,50,100\}\) correspond to 100, 200, 500, and 1000 seconds.

At sampling time \(\tau\), the exchange exposes the top \(L\) bid and ask
levels. Bid level \(f\) is the \(f\)-th highest buy price and ask level \(f\)
is the \(f\)-th lowest sell price, so
\[
  b_\tau^1 \ge b_\tau^2 \ge \cdots \ge b_\tau^L,
  \qquad
  a_\tau^1 \le a_\tau^2 \le \cdots \le a_\tau^L .
\]
With aggregate resting volumes \(v_{\tau,b}^f\) and \(v_{\tau,a}^f\), a
depth-\(L\) snapshot is
\[
  \mathcal{B}_\tau
  =
  \left\{
    (b_\tau^f,v_{\tau,b}^f),
    (a_\tau^f,v_{\tau,a}^f)
  \right\}_{f=1}^{L}.
\]
The level-1 spread is \(a_\tau^1-b_\tau^1\), and the mid-price is
\((b_\tau^1+a_\tau^1)/2\). The Coinbase raw snapshot columns are
\(\{b_\tau^f,v_{\tau,b}^f,a_\tau^f,v_{\tau,a}^f\}_{f=1}^{20}\), sampled into
10-second bins before OFI and auxiliary feature construction.

OFI mode uses 20 bid/ask order-flow coordinates plus the 11 auxiliary features
in Table~\ref{tab:features-app}. Feature extraction is per calendar day for
order-flow differences. Normalization is a causal trailing rolling
\(z\)-score over the chronological series, using a 2000-bin window. Thus a
feature at bin \(t\) uses only statistics from the past and current bin, never
future test data.

Following \citet{cont2014ofi} and \citet{kolm2023deepofi}, cumulative OFI for
level \(f\) over an interval \((t-h,t)\) is
\begin{equation}
  \mathrm{OFI}_{t-h,t}^{f}
  =
  \sum_{n=N(t-h)+1}^{N(t)}
  \left(\mathrm{bOF}_{n}^{f}-\mathrm{aOF}_{n}^{f}\right).
  \label{eq:ofi-window}
\end{equation}

All Coinbase splits are chronological: 70\% train, 15\% validation, and 15\%
test. A window is assigned to a split only when its look-back context and
future label horizon remain inside the same raw interval. This prevents
overlapping windows from leaking across train/test boundaries.
Table~\ref{tab:coinbase-datetimes-app} reports the crawl coverage and raw
test-start boundary for each crypto symbol.

For BTC, the row-index split boundaries are 2026-06-18 23:42:10 UTC for
the first training row, 2026-07-06 11:40:10 UTC for the first validation row,
and 2026-07-10 05:24:10 UTC for the first test row; the final raw timestamp is
2026-07-13 23:10:40 UTC. Because the input look-back and label horizon must
remain inside the test interval, the first BTC test label center is
2026-07-10 05:34:00 UTC. The last label centers are 2026-07-13 23:09:00,
23:07:20, 23:02:10, and 22:53:50 UTC for \(k=10,20,50,100\), respectively; in
all four cases the final future label timestamp is 2026-07-13 23:10:40 UTC.
Table~\ref{tab:btc-splits-app} gives the resulting BTC split sizes and
train-calibrated label thresholds by horizon.

\begin{table}[htbp]
\centering
\small
\begin{tabular}{@{}lrrcc@{}}
\toprule
Symbol & Snapshots & Raw start (UTC) & Raw end (UTC) & Test start (UTC) \\
\midrule
BNB & 213,679 & 2026-06-18 23:42:10 & 2026-07-13 23:10:30 & 2026-07-10 05:24:30 \\
BTC & 213,335 & 2026-06-18 23:42:10 & 2026-07-13 23:10:40 & 2026-07-10 05:24:10 \\
ETH & 213,593 & 2026-06-18 23:42:10 & 2026-07-13 23:10:30 & 2026-07-10 05:28:10 \\
USDT & 213,448 & 2026-06-18 23:42:10 & 2026-07-13 23:10:40 & 2026-07-10 05:27:10 \\
XRP & 213,752 & 2026-06-18 23:42:10 & 2026-07-13 23:10:30 & 2026-07-10 05:27:40 \\
\bottomrule
\end{tabular}
\caption{Coinbase crawl coverage and chronological test-start timestamps
computed from the normalized local parquet files. Test start denotes the raw
snapshot index at the 85\% chronological split boundary before applying
horizon-specific look-back and future-label validity constraints.}
\label{tab:coinbase-datetimes-app}
\end{table}

\begin{table}[htbp]
\centering
\small
\begin{tabular}{@{}lp{5.6cm}@{}}
\toprule
Block & Features \\
\midrule
Microstructure & spread/mid; log total bid-depth over total ask-depth;
one-step log mid-return \\
Trade flow & log buy volume; log sell volume; signed trade imbalance;
log trade count; VWAP deviation from mid \\
Activity proxies & log trade count; spread/mid; absolute one-step log
mid-return \\
\bottomrule
\end{tabular}
\caption{The 11 auxiliary features appended to the 20 OFI levels. Repeated
activity proxies are intentional substitutes for unavailable intra-bin
quote-update counts in the resampled exchange feed.}
\label{tab:features-app}
\end{table}

\begin{table}[htbp]
\centering
\small
\begin{tabular}{@{}lrrrrr@{}}
\toprule
\(k\) & Real time & \(\alpha_{\mathrm{lab}}\) & Train & Val. & Test \\
\midrule
10  & 100 s  & 0.000093 & 147,218 & 31,831 & 31,932 \\
20  & 200 s  & 0.000189 & 147,208 & 31,821 & 31,922 \\
50  & 500 s  & 0.000362 & 147,178 & 31,791 & 31,892 \\
100 & 1000 s & 0.000543 & 147,087 & 31,741 & 31,842 \\
\bottomrule
\end{tabular}
\caption{BTC split sizes and train-calibrated label thresholds. Other crypto
symbols use the same protocol with symbol-specific thresholds.}
\label{tab:btc-splits-app}
\end{table}

\subsection{How \texorpdfstring{\(\alpha_{\mathrm{lab}}\)}{alpha lab} Is Selected}
\label{app:label-alpha}

The label threshold \(\alpha_{\mathrm{lab}}\) controls the dead band between
``stationary'' and a directional move. For crypto, it is calibrated on the
training split only. We first compute the smoothed-mid trend ratio
\(r_k(\tau)\) from Eq.~\eqref{eq:label-return} over the training interval,
then set
\[
  \alpha_{\mathrm{lab}}
  =
  \mathrm{quantile}_{1/3}\left(\left|r_k(\tau)\right|:
  \tau\in\mathrm{train}\right).
\]
This gives an approximately balanced three-class split while preserving the
chronological validation and test intervals. If a config sets
\texttt{label\_alpha} to a positive number, that fixed value is used instead;
the Coinbase experiments reported here use \texttt{label\_alpha: -1},
meaning train-only calibration.

For Feishu mid-price experiments, the default label mode is quantile rather
than a symmetric fixed band. The training-split return distribution is cut at
the \(1/3\) and \(2/3\) quantiles:
\[
  \ell=\mathrm{quantile}_{1/3}(r_t:\,t\in\mathrm{train}),\qquad
  u=\mathrm{quantile}_{2/3}(r_t:\,t\in\mathrm{train}).
\]
Labels are down for \(r_t<\ell\), stationary for \(\ell\le r_t\le u\), and up
for \(r_t>u\). The same \((\ell,u)\) thresholds are then frozen for validation
and out-of-sample portfolio evaluation. A fixed \(\pm\alpha\) band is also
available, but the Feishu mid-price experiments reported here use
\texttt{label\_mode: "quantile"} and \texttt{class\_quantiles: [0.3333,
0.6667]}.

\subsection{How the Alpha-Stable Tail Index Is Selected}
\label{app:stable-alpha}

The alpha-stable parameter \(\alpha\) is separate from
\(\alpha_{\mathrm{lab}}\). It controls the tail thickness of the training
corruption: smaller values show the score head larger perturbations more
often, while \(\alpha=2\) recovers Gaussian-type behavior. We select this
quantity from training data, not from validation or test performance.

For the Coinbase configs, \(\alpha\) is based on a Hill/Pareto tail estimate
of normalized OFI magnitudes. The pooled diagnostic gives
\(\hat{\alpha}=1.62\), which motivated the heavy-tailed corruption family. The
final per-symbol configs use clipped symbol-specific estimates in the
practical range \([1.2,1.9]\), shown in Table~\ref{tab:stable-alpha-app}. BTC
therefore uses a slightly heavier-than-Gaussian setting near the observed OFI
tail, while BNB is clipped at the upper end because its estimate is closer to
light-tailed behavior under this diagnostic.

\begin{table}[htbp]
\centering
\small
\begin{tabular}{@{}lccccc@{}}
\toprule
Symbol & BTC & ETH & BNB & XRP & USDT \\
\midrule
\(\alpha\) used for \(\alpha\)-stable corruption
& 1.687 & 1.731 & 1.900 & 1.615 & 1.685 \\
\bottomrule
\end{tabular}
\caption{Coinbase alpha-stable corruption indices from train-only OFI tail
calibration.}
\label{tab:stable-alpha-app}
\end{table}

For Feishu, we do not reuse the crypto per-symbol calibration because the
task is a broad A-share transfer experiment with daily asset windows and a
different market microstructure. The Feishu alpha-stable configs therefore
use \(\alpha=1.5\) as a conservative heavy-tailed transfer default. This value
is not tuned on the out-of-sample portfolio period.

\subsection{Heavy-Tail Evidence}
\label{app:tail-evidence}

The distribution-fit diagnostics rank candidate laws by Kolmogorov-Smirnov
distance on subsampled series. Figure~\ref{fig:diagnostic-fits-main} shows
the BTC mid-price-increment and feature-PC1 distribution fits. Across all
crypto symbols, the normal reference is not the leading fit. For mid-price
increments, it ranks 11th
(BTC), 13th (BNB), 12th (ETH), 20th (XRP), and 16th (USDT) among roughly 83
candidates. For the first principal component of the full feature tensor, the
normal ranks are 31st, 65th, 25th, 40th, and 26th. Feishu shows the same
direction: normal ranks 24th for mid-price increments and 63rd for the first
feature principal component. Figure~\ref{fig:crypto-tail} gives the pooled
OFI tail evidence used to motivate the alpha-stable family.

\begin{figure}[htbp]
\centering
\begin{subfigure}{0.45\linewidth}
  \includegraphics[width=\linewidth]{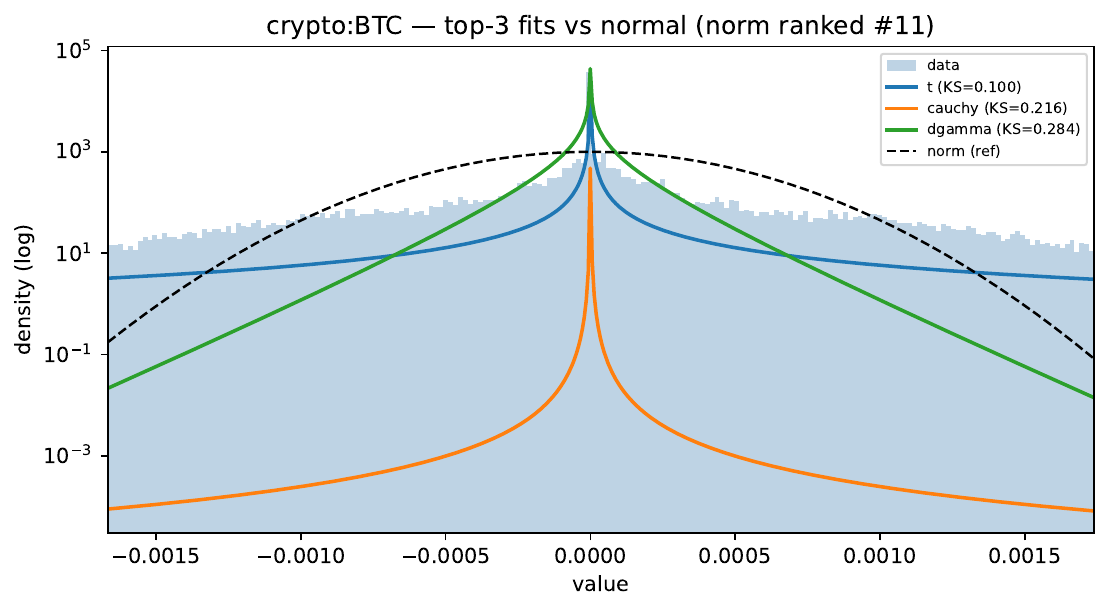}
  \caption{BTC mid-price increments}
\end{subfigure}\hfill
\begin{subfigure}{0.45\linewidth}
  \includegraphics[width=\linewidth]{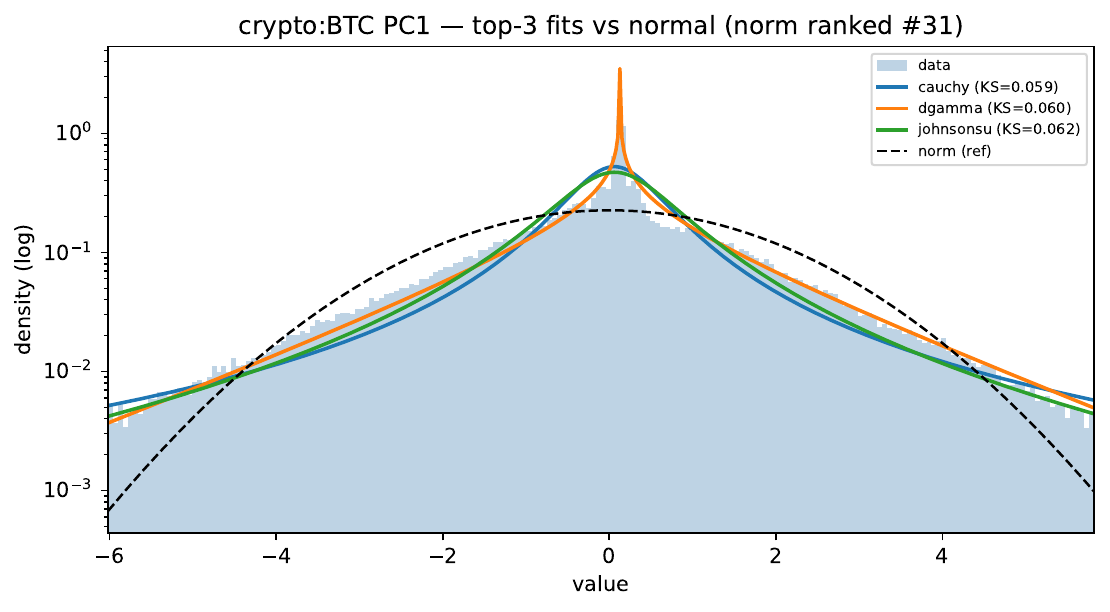}
  \caption{BTC feature PC1}
\end{subfigure}
\caption{Heavy-tail distribution-fit diagnostics. The dashed normal
reference is not the leading fit: it ranks 11th for BTC increments and 31st
for the first principal component of the feature tensor.}
\Description{Two distribution-fit plots comparing empirical BTC histograms
with heavy-tailed alternatives and a dashed normal reference.}
\label{fig:diagnostic-fits-main}
\end{figure}

\begin{figure}[htbp]
\centering
\includegraphics[height=0.28\textheight,keepaspectratio]{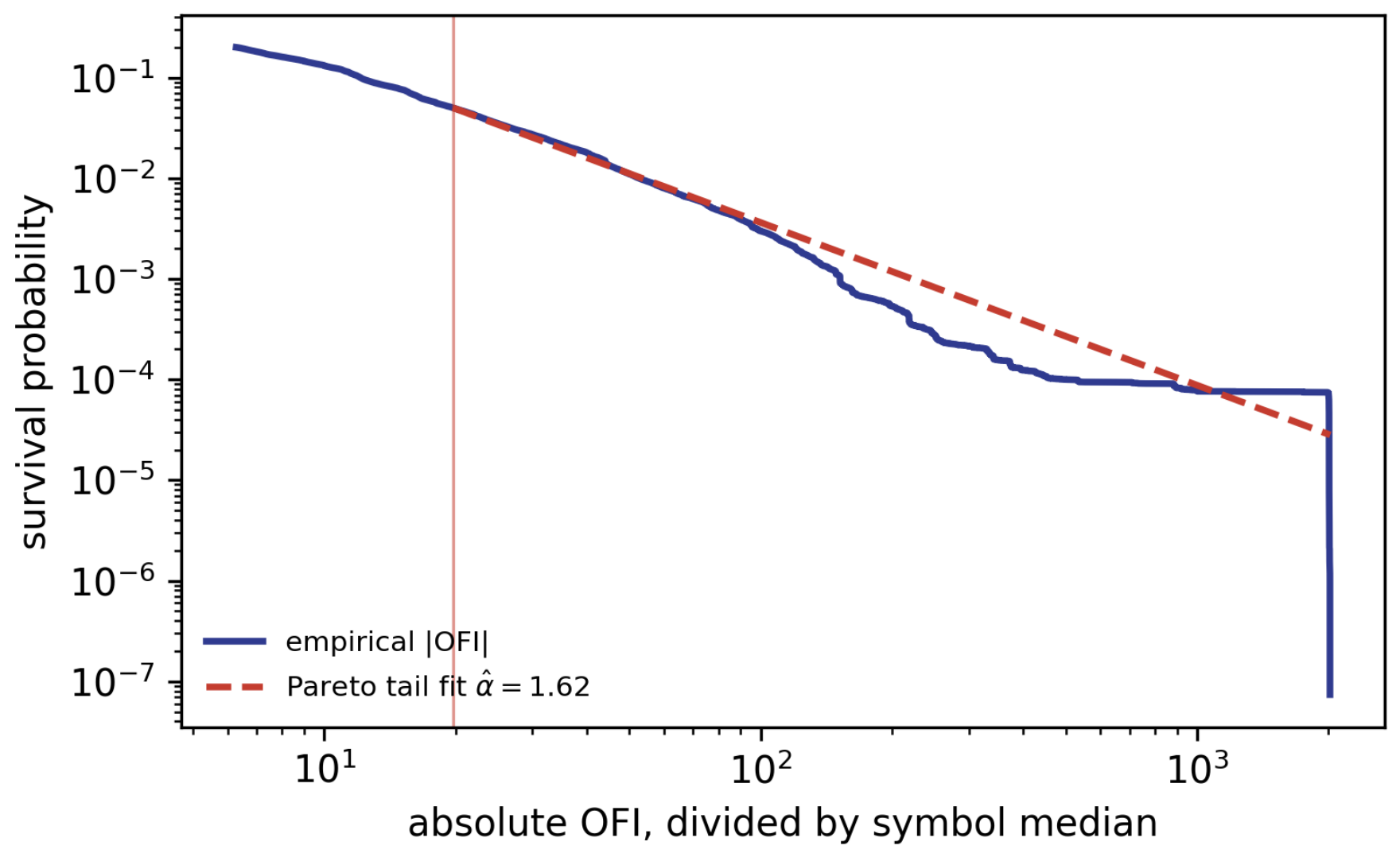}
\caption{Log-log complementary CDF of pooled absolute crypto OFI, normalized
by symbol median. The top 5\% Hill/Pareto fit gives
\(\hat{\alpha}=1.62\), supporting heavy-tailed OFI corruption.}
\Description{Log-log complementary CDF of pooled absolute OFI with a fitted
power-law tail slope corresponding to tail index 1.62.}
\label{fig:crypto-tail}
\end{figure}

\subsection{Feishu Data Details}

The Feishu dataset is the larger Chinese A-share protocol used for the
portfolio transfer experiment. The feature tensor has 259 coordinates per
asset-day. The first 240 coordinates are intraday OFI features: 24 trading-day
slots times 10 book levels. The slot grid uses the Chinese A-share continuous
trading sessions: 09:40, 09:50, ..., 11:20 in the morning and 13:00, 13:10,
..., 15:00 in the afternoon. OFI is computed at the tick level within each
asset-day, aligned to the 10-minute slot marks, and zero-filled for off-grid
5-minute bars. The OFI block is normalized with a causal 5-day rolling
\(z\)-score and clipped to \([-5,5]\).

The remaining 19 coordinates are daily OHLCV and technical variables: 1-, 5-,
10-, and 20-day returns; 5- and 20-day volatility; Amihud illiquidity; volume
\(z\)-score; 14-day RSI; distances from 5- and 20-day moving averages;
open-to-close return; high-low range; close-VWAP distance; and raw open,
close, volume, low, and high. These daily features are cross-sectionally
standardized across assets on the same day and clipped to the same numeric
range as the OFI block. Table~\ref{tab:feishu-data-app} summarizes the
Feishu universe, feature count, splits, and evaluation period.

\begin{table}[t]
\centering
\caption{Feishu data and evaluation protocol.}
\label{tab:feishu-data-app}
\footnotesize
\setlength{\tabcolsep}{3pt}
\renewcommand{\arraystretch}{0.95}
\begin{tabular}{@{}p{0.20\columnwidth}p{0.74\columnwidth}@{}}
\toprule
Block & Feishu protocol \\
\midrule
Universe & 2,306 Chinese A-share symbols \\
Rows & 1,606,720 asset-day observations \\
Features & 240 intraday OFI and 19 daily features \\
Slots & 24 ten-minute morning/afternoon slots \\
History & 50 trading days per model input \\
Splits & 756,787 train; 189,882 validation; 542,612 out-of-sample \\
Period & Two years in-sample; one year out-of-sample \\
\bottomrule
\end{tabular}
\end{table}
\section{GSM Score Table}
\label{app:gsm}

For \(u=\sqrt{W}\xi\), \(\xi\sim\mathcal{N}(0,I_d)\), and \(W>0\), the
marginal density is
\[
  \phi_t(u)=\E_W[(2\pi W)^{-d/2}\exp(-\lVert u\rVert^2/2W)].
\]
Differentiating under the integral gives
\[
  \nabla_u\phi_t(u)
  =
  -u\,\E_W[W^{-1}(2\pi W)^{-d/2}\exp(-\lVert u\rVert^2/2W)].
\]
Dividing by \(\phi_t(u)\) yields Eq.~\eqref{eq:gsm-score-main}. The table is
therefore one-dimensional in the radius even though \(d=T\cdot F=1860\) for
the crypto dataset. Appendix~\ref{app:h-score} gives the complete
posterior-precision proof and its Gaussian, jump, and alpha-stable
corollaries.

\begin{algorithm}[htbp]
\caption{\texttt{BuildScoreTable}}
\label{alg:score-table-app}
\small
\begin{algorithmic}[1]
\Require diffusion horizon \(T_{\max}\); corruption sampler \(p_t(W)\);
dimension \(d\); grid size \(n_r\); Monte Carlo size \(M\)
\Ensure per-timestep radius grids \(R_t\) and radial precision values \(H_t\)
\For{\(t=1,\dots,T_{\max}\)}
  \State draw \(W_1,\dots,W_M\sim p_t(W)\)
  \State draw \(C_1,\dots,C_M\sim\chi^2_d\)
  \State estimate typical radii \(r_j=\sqrt{W_jC_j}\)
  \State set \(r_{\max}=1.02\,\mathrm{quantile}_{0.9999}(\{r_j\})\)
  \State create \(R_t=\{r_1^\star,\dots,r_{n_r}^\star\}\) on \([0,r_{\max}]\)
  \For{each \(r^\star\in R_t\)}
    \State compute log weights
    \(g_j=-\tfrac{d}{2}\log W_j-(r^\star)^2/(2W_j)\)
    \State set \(m=\max_j g_j\)
    \State set
    \(H_t(r^\star)=
    \sum_j e^{g_j-m}W_j^{-1}/\sum_j e^{g_j-m}\)
  \EndFor
\EndFor
\State \Return \(\{(R_t,H_t)\}_{t=1}^{T_{\max}}\)
\end{algorithmic}
\end{algorithm}

\subsection{Monte Carlo Justification}
\label{app:gsm-monte-carlo}

Algorithm~\ref{alg:score-table-app} estimates \(h_t(r)\) by replacing the two
expectations in Eq.~\eqref{eq:radial-score} with Monte Carlo averages. The
argument below is pointwise in the tabulated radius: reusing one batch of
draws across the radius grid creates correlation between grid entries, but it
does not change the marginal guarantee at any fixed grid point.

Fix a timestep \(t\) and a radius \(r>0\). Here \(r=\lVert u\rVert\), so
negative values are outside the domain. We state the Monte Carlo guarantee for
positive radii only. The table may include \(r=0\) as a numerical
interpolation endpoint, but Proposition~\ref{prop:gsm-mc-accuracy} does not claim
a finite-sample bound there for the clipped alpha-stable law: the bounds in
Lemma~\ref{lem:gsm-mc-bounds} diverge as \(r\downarrow0\). If \(W_t\) is
bounded away from zero, as in the Gaussian and jump-Poisson corruptions used
here, the same ratio-estimator argument extends to \(r=0\) with bounds based
on that lower bound. For clipped alpha-stable corruption, the zero endpoint
would require separate negative-moment/concentration conditions. This is fine
because under continuous Gaussian-mixture corruption, sampled training
residuals satisfy \(\mathbb{P}(\lVert u\rVert=0)=0\): conditional on any
realization \(W_t=w>0\), the residual
\(u\mid W_t=w\sim\mathcal{N}(0,wI_d)\) is absolutely continuous, so
\[
  \mathbb{P}(\lVert u\rVert=0\mid W_t=w)
  =
  \mathbb{P}(u=0\mid W_t=w)
  =
  0.
\]
Taking expectation over \(W_t\) gives
\(\mathbb{P}(\lVert u\rVert=0)=0\). Define
\[
  g_1(w,r)=w^{-1-d/2}\exp\!\left(-\frac{r^2}{2w}\right),
  \qquad
  g_2(w,r)=w^{-d/2}\exp\!\left(-\frac{r^2}{2w}\right),
\]
and write
\[
  N_t(r)=\E_{p_t}[g_1(W,r)],
  \qquad
  D_t(r)=\E_{p_t}[g_2(W,r)].
\]
Then \(h_t(r)=N_t(r)/D_t(r)\). The empirical table entry is
\[
  \widehat h_t(r)
  =
  \frac{
    M^{-1}\sum_{i=1}^{M} g_1(W_i,r)
  }{
    M^{-1}\sum_{i=1}^{M} g_2(W_i,r)
  },
  \qquad
  W_i\stackrel{\mathrm{i.i.d.}}{\sim}p_t(W).
\]
Since \(g_1(w,r)=w^{-1}g_2(w,r)\), this is equivalently
\[
  \widehat h_t(r)
  =
  \sum_{i=1}^{M}\widehat\omega_i W_i^{-1},
  \qquad
  \widehat\omega_i
  =
  \frac{g_2(W_i,r)}{\sum_{j=1}^{M}g_2(W_j,r)}.
\]
The table value is therefore a self-normalized importance-sampling estimator
(\citealp{geweke1989bayesian}; see also \citealp[Ch.~9]{owen2013monte} for a
general treatment). The proposal is the corruption law \(p_t(W)\),
the unnormalized target is
\[
  \pi_t(dw\mid r)
  =
  \frac{g_2(w,r)p_t(dw)}{D_t(r)},
\]
and the estimated quantity is
\[
  \E_{\pi_t(\cdot\mid r)}[W^{-1}]
  =
  \frac{\E_{p_t}[W^{-1}g_2(W,r)]}{\E_{p_t}[g_2(W,r)]}
  =
  h_t(r).
\]
By Proposition~\ref{prop:ht-posterior-precision}, this is the Tweedie-type
posterior expected corruption precision
\(\E[W_t^{-1}\mid\lVert u\rVert=r]\).

\begin{lemma}[Pointwise boundedness]
\label{lem:gsm-mc-bounds}
For every \(r>0\),
\[
  \sup_{w>0} g_1(w,r)
  =
  \left(\frac{d+2}{e r^2}\right)^{1+d/2}
  =: B_1(r),
  \qquad
  \sup_{w>0} g_2(w,r)
  =
  \left(\frac{d}{e r^2}\right)^{d/2}
  =: B_2(r).
\]
Both bounds are finite for fixed \(r>0\).
\end{lemma}

\begin{proof}
Both functions have the form
\(f_a(w)=w^{-a}\exp(-r^2/2w)\), with \(a=d/2\) for \(g_2\) and
\(a=1+d/2\) for \(g_1\). Since
\[
  \frac{d}{dw}\log f_a(w)
  =
  -\frac{a}{w}+\frac{r^2}{2w^2},
\]
the unique critical point is \(w^\star=r^2/(2a)\). Also
\(f_a(w)\to0\) as \(w\to0^+\) and as \(w\to\infty\), so this critical point
is the global maximum. Evaluating \(f_a(w^\star)\) gives
\[
  f_a(w^\star)
  =
  \left(\frac{2a}{e r^2}\right)^a.
\]
Substituting \(a=d/2\) and \(a=1+d/2\) gives the two displayed bounds.
\end{proof}

This boundedness is the only fact needed to cover the corruption laws used in
the paper. For Gaussian corruption, \(p_t\) is the point mass
\(\delta_{\sigma_t^2}\). For the jump-Poisson law,
\[
  W_t=\sigma_t^2+\sum_{j=1}^{N_t}S_j\ge\sigma_t^2>0
  \quad \text{a.s.}
\]
For the clipped alpha-stable law,
\[
  W_t=\sigma_t^2\min(A,Q_{0.999}(A))>0
  \quad \text{a.s.}
\]
The clipped law may still have support arbitrarily close to zero, but for
any fixed \(r>0\) Lemma~\ref{lem:gsm-mc-bounds} makes \(g_1\) and \(g_2\)
bounded uniformly over \(w>0\). Thus \(N_t(r)\) and \(D_t(r)\) are finite, and
\(D_t(r)>0\).

\begin{proposition}[Statistical accuracy of one score-table entry]
\label{prop:gsm-mc-accuracy}
Fix \(t\) and \(r>0\), and let \(p_t\) be the Gaussian, jump-Poisson, or
clipped alpha-stable corruption law used by \textnormal{StochLOB}. Then:
\[
  \widehat h_t(r)\to h_t(r)
  \qquad
  \text{almost surely as }M\to\infty .
\]
This is the usual consistency property of self-normalized importance-sampling
estimators \citep{geweke1989bayesian}. Moreover, for any \(\delta\in(0,1)\), if
\[
  M
  \ge
  \frac{2B_2(r)^2\log(4/\delta)}{D_t(r)^2},
\]
then with probability at least \(1-\delta\),
\[
  \left|\widehat h_t(r)-h_t(r)\right|
  \le
  \frac{2\{B_1(r)+h_t(r)B_2(r)\}}{D_t(r)}
  \sqrt{\frac{\log(4/\delta)}{2M}}.
\]
This gives a finite-sample guarantee of the kind developed for
importance-sampling estimators more generally \citep{agapiou2017importance}.
Finally,
\[
  \sqrt{M}\bigl(\widehat h_t(r)-h_t(r)\bigr)
  \xrightarrow{d}
  \mathcal N\!\left(0,\sigma_t^2(r)\right),
\]
This is the corresponding asymptotic normality result for this estimator class
(\citealp{geweke1989bayesian}; \citealp[Ch.~9]{owen2013monte}), where
\[
  \sigma_t^2(r)
  =
  \frac{1}{D_t(r)^2}
  \operatorname{Var}_{p_t}\!\left[
    g_2(W,r)\bigl(W^{-1}-h_t(r)\bigr)
  \right].
\]
For deterministic Gaussian corruption this variance is zero, recovering the
exact value \(h_t(r)=1/\sigma_t^2\).
\end{proposition}

\begin{proof}
Lemma~\ref{lem:gsm-mc-bounds} gives
\(\E_{p_t}|g_1(W,r)|<\infty\) and \(\E_{p_t}|g_2(W,r)|<\infty\). The strong
law of large numbers therefore gives
\[
  M^{-1}\sum_{i=1}^M g_1(W_i,r)\to N_t(r),
  \qquad
  M^{-1}\sum_{i=1}^M g_2(W_i,r)\to D_t(r)
\]
almost surely. Since \(D_t(r)>0\), the continuous mapping theorem gives
\(\widehat h_t(r)\to N_t(r)/D_t(r)=h_t(r)\) almost surely.

For the finite-sample statement, write
\[
  \widehat N_t=M^{-1}\sum_i g_1(W_i,r),
  \qquad
  \widehat D_t=M^{-1}\sum_i g_2(W_i,r).
\]
Algebra gives
\[
  \left|\widehat h_t(r)-h_t(r)\right|
  \le
  \frac{
    |\widehat N_t-N_t(r)|
    +h_t(r)|\widehat D_t-D_t(r)|
  }{\widehat D_t}.
\]
By Lemma~\ref{lem:gsm-mc-bounds}, \(g_1(W_i,r)\in[0,B_1(r)]\) and
\(g_2(W_i,r)\in[0,B_2(r)]\). Hoeffding's inequality
\citep{hoeffding1963probability} and a union bound imply that, with
probability at least \(1-\delta\),
\[
  |\widehat N_t-N_t(r)|
  \le
  B_1(r)\sqrt{\frac{\log(4/\delta)}{2M}},
  \qquad
  |\widehat D_t-D_t(r)|
  \le
  B_2(r)\sqrt{\frac{\log(4/\delta)}{2M}}.
\]
The stated lower bound on \(M\) ensures
\(\widehat D_t\ge D_t(r)/2\) on this event. Substitution gives the displayed
non-asymptotic error bound.

For asymptotic normality, apply the multivariate central limit theorem to
\[
  Y_i=(g_1(W_i,r),g_2(W_i,r)).
\]
Its covariance is finite by Lemma~\ref{lem:gsm-mc-bounds}. The delta method
\citep{vandervaart1998asymptotic} applied to
\(\varphi(x,y)=x/y\), whose gradient at \((N_t(r),D_t(r))\) is
\[
  \nabla\varphi
  =
  \left(\frac{1}{D_t(r)},-\frac{h_t(r)}{D_t(r)}\right),
\]
yields an asymptotic variance
\[
  \frac{1}{D_t(r)^2}
  \operatorname{Var}_{p_t}\!\left[
    g_1(W,r)-h_t(r)g_2(W,r)
  \right].
\]
Using \(g_1(W,r)=W^{-1}g_2(W,r)\) gives the stated closed form.
\end{proof}

The proposition shows that the Monte Carlo part of the table has the usual
\(M^{-1/2}\) pointwise error scale, with a variance determined by how much
posterior-weighted precision \(W^{-1}\) varies under the corruption law. The
log-sum-exp computation in Algorithm~\ref{alg:score-table-app} implements the
same ratio estimator in a numerically stable form. Linear interpolation and
clamping are deterministic post-processing steps applied after the grid
values have been estimated; they do not change the pointwise Monte Carlo
claim above, although they add the ordinary discretization error of a finite
radius grid.

StochLOB (\mbox{L\'evy}) uses \(n_r=512,M=20{,}000\). StochLOB
(\(\alpha\)-stable) uses \(n_r=256,M=8192\). Training radii are linearly
interpolated and clamped beyond the largest tabulated radius. The Gaussian
limit \(W\equiv\sigma_t^2\) recovers \(h_t(r)=1/\sigma_t^2\), so no table is
needed for a purely Gaussian baseline.

\section{Experimental Setup and Implementation Details}
\label{app:models}

\subsection{Shared Architecture}

Both \textsc{StochLOB} variants use optional BiN normalization, a two-layer
bidirectional GRU with 64 hidden units per direction, one DiT-style temporal
self-attention block with four heads and adaLN-Zero timestep conditioning,
attention pooling for trend classification, and a small convolutional score
head. The corruption law and loss differ; the inference path without
stochastic corruption is the same. Figure~\ref{fig:model-architecture} shows
the workflow in the main text.

\subsection{Optimization}

Optimization is fixed across the proposed models: AdamW
\citep{loshchilov2019adamw}, learning rate \(10^{-4}\), default moments
\((\beta_1,\beta_2)=(0.9,0.999)\), \(\epsilon=10^{-8}\), weight decay
\(3\times10^{-4}\), and a per-update schedule with 500 linear warmup steps
followed by cosine decay over the planned training steps. We also use
gradient clipping at 1.0, label smoothing 0.1, and dropout 0.2. Coinbase runs
use batch size 64 and at most 80 epochs with patience 15. Feishu runs use
batch size 256 and at most 30 epochs. Checkpoints are selected by validation
macro-F1 on unperturbed windows, not by the score loss, because the final
deployed model is a classifier. The diffusion auxiliary weight is
\(\lambda_{\mathrm{diff}}=1.0\). The jump-Poisson setting uses a
variance-preserving schedule with \(T_{\max}=1000\) and gamma jump
shape/scale 1. The alpha-stable setting clips the positive stable scale at
the 0.999 quantile.

All reported Coinbase configuration files set \texttt{device: "cuda"} and use
single-GPU workers. The reported training runs and checkpoint
artifact repair/validation were carried out on a single NVIDIA GeForce RTX
4090 worker with 24 GB VRAM. The StochLOB
classifier trunk has about 0.53M trainable parameters. Logged per-sample
forward costs for the supervised baselines range from effectively zero for
linear models to about 0.122 GFLOPs for Axial-LOB; DeepLOB and TLOB are logged
at about 0.093 and 0.061 GFLOPs/sample, respectively. The score-table Monte
Carlo preprocessing is performed once per corruption law and schedule and is
reused by the training runs.

\subsection{Crypto Evaluation}

For each symbol and horizon, all baselines use the same chronological
train/validation/test windows, the same feature normalization, and the same
train-calibrated label thresholds. Standard held-out metrics are accuracy and
macro-F1 on the held-out test interval. The main paper reports BTC; the
appendix reports ETH, BNB, XRP, and USDT. Seed-sweep results are reported as
mean and sample standard deviation over seeds \(13,21,42,84,168\).

The realized stress evaluation freezes the trained checkpoint and partitions
held-out BTC windows by statistics computed from the observed test inputs. The
stress score is the maximum one-step level-averaged feature increment divided
by realized volatility; the variance score is the realized variance of those
increments; and the jump-residual score is the residual from regressing the
largest increment on realized variance in log-log space. This keeps the data
split, trained weights, class thresholds, and decision rule fixed while
testing the regimes where stochastic score-matching corruption should matter.

\subsection{Feishu Evaluation}

Feishu uses a two-year in-sample period and one-year out-of-sample release
period. The in-sample period is split chronologically per asset into
train/validation windows. The out-of-sample period is never used for label
thresholds, normalization decisions, checkpoint selection, or trading
threshold selection. Model probabilities are exported day by day and consumed
by the common portfolio engine described in Appendix~\ref{app:feishu-portfolio}.

\section{Training and Inference Pseudocode}
\label{app:algorithms}

Algorithms~\ref{alg:score-table-app}, \ref{alg:training-main}, and
\ref{alg:inference-main} give the operational recipe. The score table is
constructed once per corruption law and schedule; training then uses table
interpolation for each sampled residual norm.

\begin{algorithm}[htbp]
\caption{\normalfont\textsc{TrainNonGaussianLOB}}
\label{alg:training-main}
\small
\begin{algorithmic}[1]
\Require training windows \(\mathcal{D}\); classifier \(f_\theta\); score head
\(s_\theta\); corruption sampler \(p_t(W)\); score table \(h_t\); loss weight
\(\lambda_{\mathrm{diff}}\)
\Ensure checkpoint with best validation macro-F1
\For{epoch \(=1,\dots,E\)}
  \For{minibatch \((x_0,y)\sim\mathcal{D}\)}
    \State compute unperturbed logits \(z_0=f_\theta(x_0)\)
    \State compute \(\mathcal{L}_{\mathrm{cls}}=\CE(z_0,y)\)
    \State sample diffusion step \(t\sim\mathrm{Uniform}\{1,\dots,T_{\max}\}\)
    \State sample variance \(W\sim p_t(W)\) and noise
    \(\xi\sim\mathcal{N}(0,I_d)\)
    \State form \(u=\sqrt{W}\xi\) and \(x_t=a_tx_0+u\)
    \State interpolate \(h_t(\lVert u\rVert)\) from the score table
    \State compute target score \(g^\star=-u\,h_t(\lVert u\rVert)\)
    \State compute
    \(\mathcal{L}_{\mathrm{score}}=
    \lVert s_\theta(x_t,t)-g^\star\rVert^2_{\mathrm{mean}}\)
    \State update \(\theta\) using
    \(\mathcal{L}_{\mathrm{cls}}+
    \lambda_{\mathrm{diff}}\mathcal{L}_{\mathrm{score}}\)
  \EndFor
  \State evaluate validation macro-F1 on unperturbed windows and save the best
  checkpoint
\EndFor
\end{algorithmic}
\end{algorithm}

\begin{algorithm}[htbp]
\caption{\normalfont\textsc{InferStochLOB}}
\label{alg:inference-main}
\small
\begin{algorithmic}[1]
\Require normalized LOB window \(x_0\); trained classifier path \(f_\theta\)
\Ensure class probabilities for down, stationary, and up
\State pass \(x_0\) through the same causal normalization used in training
\State encode the unperturbed window through BiGRU, temporal attention, and
pooling
\State compute trend logits \(z=f_\theta(x_0)\)
\State return \(p(y\mid x_0)=\mathrm{softmax}(z)\)
\Statex \textit{The corruption sampler, score table, score head, and reverse
diffusion process are not used at inference.}
\end{algorithmic}
\end{algorithm}

\section{Additional Crypto Results}
\label{app:results}

The main paper keeps the Coinbase exchange table focused on BTC. This
appendix reports the remaining coins so the reader can see whether the same
story survives beyond one crypto pair. Table~\ref{tab:crypto-altcoin-app}
reports ETH, BNB, XRP, and USDT across all horizons. Within each symbol and
horizon, bold marks the best accuracy or macro-F1 value and underlining marks
the second-best value under the same standard held-out test protocol as the
main table.
A partial rule separates the two proposed non-Gaussian models from the
standard baselines. The broader pattern is deliberately mixed: linear and
attention baselines remain strong on several standard benchmark cells, while
StochLOB (alpha-stable) wins the BNB short-horizon cell and is competitive on ETH,
XRP, and USDT. StochLOB (jump-Poisson) is frequently near the top, consistent
with the main claim that market-shaped score corruption need not sacrifice
standard predictive quality to improve stress-regime behavior.
Table~\ref{tab:perclass-app}
provides the per-class precision, recall, and F1 breakdown for BTC at
$k{=}10$, illustrating how each model distributes errors among the three
trend classes. Realized BTC stress-window diagnostics are reported in the main
paper in Figure~\ref{fig:stress-main} and Table~\ref{tab:stress-extreme-metrics};
the complete stress-regime grid is reported in
Table~\ref{tab:stress-extreme-full-app}.

\begin{table*}[htbp]
\centering
\scriptsize
\setlength{\tabcolsep}{3pt}
\renewcommand{\arraystretch}{1.06}
\resizebox{\textwidth}{!}{%
\begin{tabular}{@{}lclcccccr@{}}
\toprule
Regime & \(k\) & Best heavy-tailed StochLOB & StochLOB F1 & Gaussian F1 &
Best supervised & Gain vs Gaussian & Gain vs supervised & \(n_{\mathrm{tail}}\)\\
\midrule
Stress & 10 & alpha-stable & \textbf{0.712} & 0.703 & TLOB 0.708 & +0.008 & +0.004 & 3194\\
Stress & 20 & alpha-stable & 0.655 & \textbf{0.663} & LinVAR 0.648 & -0.007 & +0.007 & 3193\\
Stress & 50 & jump-Poisson & 0.610 & 0.592 & TLOB \textbf{0.611} & +0.019 & -0.001 & 3190\\
Stress & 100 & alpha-stable & \textbf{0.647} & 0.604 & LogReg 0.614 & +0.043 & +0.034 & 3185\\
\midrule
High variance & 10 & alpha-stable & 0.695 & 0.665 & TLOB \textbf{0.697} & +0.031 & -0.002 & 3194\\
High variance & 20 & jump-Poisson & \textbf{0.686} & 0.680 & LinVAR 0.651 & +0.006 & +0.034 & 3193\\
High variance & 50 & alpha-stable & \textbf{0.658} & 0.603 & LinVAR 0.645 & +0.055 & +0.014 & 3190\\
High variance & 100 & alpha-stable & \textbf{0.659} & 0.613 & LinVAR 0.628 & +0.046 & +0.030 & 3185\\
\midrule
Jump residual & 10 & alpha-stable & \textbf{0.696} & 0.679 & TLOB 0.693 & +0.017 & +0.003 & 3194\\
Jump residual & 20 & jump-Poisson & 0.643 & \textbf{0.651} & BiN-CTABL 0.622 & -0.008 & +0.021 & 3193\\
Jump residual & 50 & jump-Poisson & 0.602 & 0.565 & TLOB \textbf{0.607} & +0.036 & -0.005 & 3190\\
Jump residual & 100 & alpha-stable & \textbf{0.637} & 0.567 & LinVAR 0.597 & +0.070 & +0.039 & 3185\\
\bottomrule
\end{tabular}}
\caption{Complete BTC extreme-window macro-F1 grid underlying
Table~\ref{tab:stress-extreme-metrics}. Each row uses the top 10\% held-out
windows for the named regime and horizon. The table reports the better of the
two heavy-tailed StochLOB variants, the matched Gaussian StochLOB baseline, and
the best supervised comparator in the same tail subset. \textbf{Best F1 in each
row is shown in bold.} Rows with negative gains are retained rather than
filtered out.}
\label{tab:stress-extreme-full-app}
\end{table*}

\begin{table*}[t]
\centering
\tiny
\setlength{\tabcolsep}{2pt}
\renewcommand{\arraystretch}{1.04}
\begin{tabular}{@{}llcccccccc@{}}
\toprule
& & \multicolumn{2}{c}{$k=10$} & \multicolumn{2}{c}{$k=20$} & \multicolumn{2}{c}{$k=50$} & \multicolumn{2}{c}{$k=100$} \\
\cmidrule(lr){3-4} \cmidrule(lr){5-6} \cmidrule(lr){7-8} \cmidrule(l){9-10}
Symbol & Model & Acc & F1 & Acc & F1 & Acc & F1 & Acc & F1 \\
\midrule
ETH & LogReg & 0.6669 & 0.6653 & 0.6304 & 0.6334 & 0.5975 & 0.5999 & 0.5573 & 0.5614\\
ETH & LinVAR & 0.6913 & 0.6876 & 0.6518 & 0.6550 & \textbf{0.6033} & \textbf{0.6074} & 0.5686 & 0.5734\\
ETH & DeepLOB & 0.6615 & 0.6543 & 0.5752 & 0.5754 & 0.5471 & 0.5464 & 0.4793 & 0.4829\\
ETH & BiN-CTABL & 0.6857 & 0.6820 & 0.6467 & 0.6484 & 0.5771 & 0.5792 & 0.5464 & 0.5501\\
ETH & TLOB & 0.6840 & 0.6793 & 0.6389 & 0.6412 & 0.5606 & 0.5624 & \textbf{0.5826} & \textbf{0.5782}\\
ETH & OFSATNet & \underline{0.6987} & \textbf{0.6951} & \underline{0.6538} & \underline{0.6562} & 0.5861 & 0.5875 & \underline{0.5715} & \underline{0.5769}\\
\cmidrule(l){2-10}
ETH & StochLOB (jump-Poisson) & 0.6889 & 0.6851 & 0.6505 & 0.6525 & \underline{0.5992} & \underline{0.6017} & 0.5592 & 0.5640\\
ETH & StochLOB (alpha-stable) & \textbf{0.7001} & \underline{0.6950} & \textbf{0.6573} & \textbf{0.6593} & 0.5915 & 0.5939 & 0.5712 & 0.5737\\
\midrule
BNB & LogReg & 0.6730 & 0.6725 & 0.6320 & 0.6343 & \underline{0.5803} & \underline{0.5837} & \textbf{0.5352} & \textbf{0.5378}\\
BNB & LinVAR & 0.6863 & 0.6855 & \textbf{0.6467} & \textbf{0.6493} & \textbf{0.5822} & \textbf{0.5860} & \underline{0.5334} & \underline{0.5366}\\
BNB & DeepLOB & 0.6405 & 0.6357 & 0.5547 & 0.5533 & 0.4515 & 0.4472 & 0.4020 & 0.3787\\
BNB & BiN-CTABL & 0.6765 & 0.6771 & 0.6068 & 0.6084 & 0.5446 & 0.5469 & 0.4804 & 0.4833\\
BNB & TLOB & 0.6856 & 0.6858 & 0.6247 & 0.6252 & 0.5510 & 0.5481 & 0.5049 & 0.5065\\
BNB & OFSATNet & 0.7012 & 0.7005 & 0.6311 & 0.6327 & 0.5561 & 0.5582 & 0.5039 & 0.5052\\
\cmidrule(l){2-10}
BNB & StochLOB (jump-Poisson) & \underline{0.7019} & \underline{0.7006} & \underline{0.6356} & \underline{0.6367} & 0.5742 & 0.5768 & 0.5086 & 0.5113\\
BNB & StochLOB (alpha-stable) & \textbf{0.7096} & \textbf{0.7086} & 0.6227 & 0.6244 & 0.5660 & 0.5691 & 0.5057 & 0.5093\\
\midrule
XRP & LogReg & 0.6485 & 0.6492 & 0.6167 & 0.6197 & \underline{0.6114} & \underline{0.6162} & 0.5764 & \underline{0.5822}\\
XRP & LinVAR & 0.6686 & 0.6677 & \underline{0.6238} & \underline{0.6268} & \textbf{0.6173} & \textbf{0.6218} & \textbf{0.5870} & \textbf{0.5931}\\
XRP & DeepLOB & 0.6027 & 0.6024 & 0.5292 & 0.5301 & 0.4922 & 0.4868 & 0.4818 & 0.4855\\
XRP & BiN-CTABL & 0.6489 & 0.6496 & 0.6076 & 0.6091 & 0.5729 & 0.5764 & 0.5495 & 0.5519\\
XRP & TLOB & 0.6606 & 0.6587 & 0.6164 & 0.6129 & 0.5759 & 0.5808 & 0.5391 & 0.5445\\
XRP & OFSATNet & \textbf{0.6758} & \textbf{0.6740} & \textbf{0.6290} & \textbf{0.6298} & 0.6092 & 0.6080 & 0.5423 & 0.5484\\
\cmidrule(l){2-10}
XRP & StochLOB (jump-Poisson) & 0.6508 & 0.6515 & 0.6086 & 0.6108 & 0.6052 & 0.6071 & 0.5495 & 0.5550\\
XRP & StochLOB (alpha-stable) & \underline{0.6739} & \underline{0.6731} & 0.6203 & 0.6226 & 0.6010 & 0.6011 & \underline{0.5830} & 0.5809\\
\midrule
USDT & LogReg & 0.6298 & 0.6315 & 0.5956 & 0.5974 & \underline{0.5735} & 0.5716 & \underline{0.5467} & \underline{0.5425}\\
USDT & LinVAR & 0.6339 & 0.6360 & \textbf{0.6023} & \textbf{0.6047} & \textbf{0.5761} & \textbf{0.5738} & \textbf{0.5524} & \textbf{0.5484}\\
USDT & DeepLOB & 0.5903 & 0.5888 & 0.5376 & 0.5363 & 0.5309 & 0.5160 & 0.5211 & 0.5161\\
USDT & BiN-CTABL & 0.6249 & 0.6266 & 0.5768 & 0.5792 & 0.5508 & 0.5504 & 0.5107 & 0.5110\\
USDT & TLOB & 0.6324 & 0.6330 & 0.5817 & 0.5827 & 0.5500 & 0.5535 & 0.5256 & 0.5196\\
USDT & OFSATNet & \underline{0.6364} & \underline{0.6378} & 0.5932 & 0.5944 & 0.5644 & 0.5563 & 0.5320 & 0.5318\\
\cmidrule(l){2-10}
USDT & StochLOB (jump-Poisson) & 0.6283 & 0.6290 & 0.5881 & 0.5884 & 0.5710 & \underline{0.5721} & 0.5329 & 0.5274\\
USDT & StochLOB (alpha-stable) & \textbf{0.6475} & \textbf{0.6493} & \underline{0.5990} & \underline{0.6007} & 0.5635 & 0.5687 & 0.5369 & 0.5401\\
\bottomrule
\end{tabular}
\caption{Additional crypto held-out test Accuracy (Acc) and Macro-F1 (F1)
under the same model scope as the main table. Bold marks the best value
within each symbol and horizon; underlining marks the second-best value.}
\label{tab:crypto-altcoin-app}
\end{table*}

\begin{table}[htbp]
\centering
\small
\begin{tabular}{@{}lccc@{\hspace{4mm}}ccc@{}}
\toprule
& \multicolumn{3}{c}{StochLOB (jump-Poisson)}
& \multicolumn{3}{c}{StochLOB (alpha-stable)}\\
\cmidrule(lr){2-4}\cmidrule(l){5-7}
Class ($k=10$) & P & R & F1 & P & R & F1\\
\midrule
down & 0.7023 & 0.6188 & 0.6579 & 0.7516 & 0.6052 & 0.6705\\
stationary & 0.6685 & 0.7389 & 0.7019 & 0.6646 & 0.7655 & 0.7115\\
up & 0.6805 & 0.6679 & 0.6741 & 0.6805 & 0.6794 & 0.6800\\
\bottomrule
\end{tabular}
\caption{BTC per-class precision (P), recall (R), and F1 at $k=10$.}
\label{tab:perclass-app}
\end{table}

\section{Feishu Portfolio Management Details}
\label{app:feishu-portfolio}

\subsection{Models and Prediction Targets}
\label{app:feishu-model-targets}

Every Feishu model is trained to predict the same three-class next-period
mid-price trend target:
\[
  y\in\{\mathrm{down},\mathrm{flat},\mathrm{up}\}.
\]
The portfolio experiment does not introduce a separate trading model. It uses
the classifier probabilities
\[
  (P(\mathrm{down}),P(\mathrm{flat}),P(\mathrm{up}))
\]
as alpha signals. This makes the downstream transfer test direct: if the
learned probabilities contain useful directional information, a constrained
portfolio rule should be able to convert them into positive out-of-sample
performance. The experiment is therefore a check of probability usability,
not a claim that the paper designs a deployable trading system.

\subsection{Portfolio Construction}
\label{app:feishu-portfolio-setup}

The Feishu portfolio experiment uses the out-of-sample release period only for
the exploratory portfolio evaluation. In-sample data are used as causal
warmup history for rolling features such as returns, volatility, RSI, and the
5-day OFI normalizer. For each model and each trading day $t$, the signal is
computed from the historical window ending at $t-1$. The portfolio rule receives
$(P(\mathrm{down}),P(\mathrm{flat}),P(\mathrm{up}))$ and converts the
probabilities into positions. Assets with sufficiently high up probability
are eligible for buying, while assets with sufficiently high down probability
are eligible for partial selling.

The backtest uses an initial capital of RMB 50 million. The same constraints
are applied to every model: 100-share lot size, commission rate $10^{-4}$,
stamp-duty rate $5\times10^{-4}$ on sells, RMB 5 minimum commission, T+1 sell
lock, a 5\% cash buffer, and a minimum-holdings target of 10 names. These
rules matter for Chinese A-share data. A signal that is statistically useful
but unstable can create high turnover and cost drag; a signal that is too
conservative can fail to trade. The portfolio experiment therefore measures
whether the model output remains usable after realistic market constraints.

The daily rule has three steps. First, held assets are examined for exit. A
held asset is fully sold when $P(\mathrm{down})>0.85$ and the predicted class
is down, subject to the T+1 sell lock and the minimum-holdings constraint. If a
held asset is predicted flat and was not already fully sold, the strategy
sells 50\% of the sellable shares. Second, new buys are selected from assets
with $P(\mathrm{up})>0.75$ and predicted class up. Third, available capital
after the 5\% cash buffer is allocated equally across the selected buy
candidates. Orders are rounded to 100-share lots, buy executions use the
morning VWAP entry price, and sell executions use the configured close-price
sell mode. The model therefore controls which assets are traded, while the
portfolio engine keeps sizing and market rules common across all methods.

We also include a passive equal-weight index. On each out-of-sample day it
holds all valid assets equally and applies the same one-day
VWAP-entry-to-close return convention. It is not an active trading strategy
and it does not pay transaction costs in this implementation. Its purpose is
to show whether a learned signal adds value relative to broad market exposure.

\subsection{Out-of-Sample Portfolio Metrics}
\label{app:feishu-portfolio-results}

The main out-of-sample portfolio metrics are reported in
Table~\ref{tab:feishu-portfolio-main}, and the cumulative-return comparison is
shown in Figure~\ref{fig:feishu-portfolio-main}. We keep the numeric table in
the main text because it is the economic summary most directly connected to
the experimental story. This appendix gives the construction details needed to
interpret those numbers.

The most important implementation point is that the portfolio engine is shared
across models. A model can improve the backtest only by changing the
probability triplet used for buy, hold, and sell decisions; it cannot change
position sizing, execution prices, transaction-cost assumptions, or market
constraints. In the reported run, several evaluated architectures produced
mostly flat predictions during the out-of-sample window and therefore did not
produce meaningful active portfolios under the common thresholds. These
inactive runs are omitted from Table~\ref{tab:feishu-portfolio-main} so that
the main comparison focuses on strategies that actually trade.

\section{Posterior-Precision Interpretation of \texorpdfstring{\(h_t\)}{h-t}}
\label{app:h-score}

\begin{proposition}[Posterior-precision interpretation for isotropic Gaussian scale mixtures]
\label{prop:ht-posterior-precision}
Fix a noise index $t$. Suppose $W_t>0$ almost surely and
\[
    \mathbb E\!\left[W_t^{-d/2-2}\right] < \infty.
\]
Let $r=\lVert u\rVert$. Then
\[
    h_t(r)
    = \mathbb E\!\left[W_t^{-1}\mid u\right]
    = \mathbb E\!\left[W_t^{-1}\mid \lVert u\rVert=r\right].
\]
Moreover,
\[
    h_t'(r)
    =-r\,\operatorname{Var}\!\left(
        W_t^{-1}\mid \lVert u\rVert=r
    \right)
    \le 0.
\]
Hence $h_t(r)$ is the posterior expected corruption precision and is
non-increasing in the residual magnitude. In particular, residual adaptation
is a property of the random scale-mixture family itself: it is absent when
$W_t$ is deterministic and is strict when the scale distribution is
non-degenerate.
\end{proposition}

\begin{proof}
For fixed $w$, write
\[
    \phi_w(u)
    =\mathcal N(u;0,wI_d)
    =(2\pi w)^{-d/2}
      \exp\!\left(-\frac{\lVert u\rVert^2}{2w}\right).
\]
The component score is
\[
    \nabla_u\log\phi_w(u)=-\frac{u}{w},
    \qquad\text{hence}\qquad
    \nabla_u\phi_w(u)=-\frac{u}{w}\phi_w(u).
\]
Since
\[
    q_t(u)=\mathbb E\!\left[\phi_{W_t}(u)\right],
\]
the condition
\[
    \mathbb E\!\left[W_t^{-d/2-2}\right] < \infty
\]
provides sufficient integrability to interchange differentiation and
expectation in the expressions below. Therefore,
\[
    \nabla_u q_t(u)
    =
    -u\,\mathbb E\!\left[W_t^{-1}\phi_{W_t}(u)\right].
\]
Dividing by $q_t(u)$ gives
\[
    \nabla_u\log q_t(u)
    =
    -u\,
    \frac{
        \mathbb E\!\left[W_t^{-1}\phi_{W_t}(u)\right]
    }{
        \mathbb E\!\left[\phi_{W_t}(u)\right]
    }.
\]
By Bayes' rule in expectation form,
\[
    \mathbb E\!\left[f(W_t)\mid u\right]
    =
    \frac{
        \mathbb E\!\left[f(W_t)\phi_{W_t}(u)\right]
    }{
        \mathbb E\!\left[\phi_{W_t}(u)\right]
    }
\]
for any integrable function $f$. Taking $f(w)=w^{-1}$ therefore gives
\[
    \nabla_u\log q_t(u)
    =
    -u\,\mathbb E\!\left[W_t^{-1}\mid u\right].
\]
Comparing this identity with
\[
    \nabla_u\log q_t(u)
    =-u\,h_t(\lVert u\rVert)
\]
proves the first equality. Because the Gaussian components are isotropic,
the posterior depends on $u$ only through $r=\lVert u\rVert$, giving
\[
    h_t(r)
    =\mathbb E\!\left[
        W_t^{-1}\mid \lVert u\rVert=r
    \right].
\]

To establish monotonicity, let
\[
    s=\frac{r^2}{2}
\]
and define
\[
    Z_t(s)
    =
    \mathbb E\!\left[
        W_t^{-d/2}
        \exp\!\left(-\frac{s}{W_t}\right)
    \right].
\]
Since
\[
    q_t(u)=(2\pi)^{-d/2}Z_t(s),
\]
Eq.~\eqref{eq:radial-score} can be written as
\[
    h_t(r)
    =
    -\frac{\partial}{\partial s}\log Z_t(s).
\]
Indeed,
\[
    -\frac{\partial}{\partial s}\log Z_t(s)
    =
    \frac{
        \mathbb E\!\left[
            W_t^{-d/2-1}\exp(-s/W_t)
        \right]
    }{
        \mathbb E\!\left[
            W_t^{-d/2}\exp(-s/W_t)
        \right]
    }
    =
    h_t(r).
\]

The same Bayes weighting gives, for $k=1,2$,
\[
    \mathbb E\!\left[W_t^{-k}\mid \lVert u\rVert=r\right]
    =
    \frac{
        \mathbb E\!\left[
            W_t^{-d/2-k}\exp(-s/W_t)
        \right]
    }{
        Z_t(s)
    }.
\]
Differentiating $Z_t$ with respect to $s$ gives
\[
    Z_t'(s)
    =
    -\mathbb E\!\left[
        W_t^{-d/2-1}\exp(-s/W_t)
    \right]
\]
and
\[
    Z_t''(s)
    =
    \mathbb E\!\left[
        W_t^{-d/2-2}\exp(-s/W_t)
    \right].
\]
Hence
\[
    -\frac{Z_t'(s)}{Z_t(s)}
    =
    \mathbb E\!\left[
        W_t^{-1}\mid \lVert u\rVert=r
    \right]
\]
and
\[
    \frac{Z_t''(s)}{Z_t(s)}
    =
    \mathbb E\!\left[
        W_t^{-2}\mid \lVert u\rVert=r
    \right].
\]
Now,
\[
    \frac{\partial^2}{\partial s^2}\log Z_t(s)
    =
    \frac{Z_t''(s)}{Z_t(s)}
    -
    \left(
        \frac{Z_t'(s)}{Z_t(s)}
    \right)^2,
\]
so
\[
    \frac{\partial^2}{\partial s^2}\log Z_t(s)
    =
    \operatorname{Var}\!\left(
        W_t^{-1}\mid \lVert u\rVert=r
    \right).
\]
Since
\[
    h_t(r)
    =
    -\frac{\partial}{\partial s}\log Z_t(s),
\]
it follows that
\[
    \frac{\partial h_t}{\partial s}
    =
    -\operatorname{Var}\!\left(
        W_t^{-1}\mid \lVert u\rVert=r
    \right)
    \le 0.
\]
Finally, using $s=r^2/2$,
\[
    h_t'(r)
    =
    r\frac{\partial h_t}{\partial s}
    =
    -r\,\operatorname{Var}\!\left(
        W_t^{-1}\mid \lVert u\rVert=r
    \right)
    \le 0.
\]
If $W_t$ is non-degenerate, the Gaussian likelihood is strictly positive
for every $W_t>0$, so posterior weighting cannot collapse the scale to a
single value. The posterior variance is therefore strictly positive and
$h_t'(r)<0$ for $r>0$. Thus $h_t$ is strictly decreasing for every
genuinely random scale mixture.
\end{proof}

Proposition~\ref{prop:ht-posterior-precision} provides an intuitive
interpretation of the function $h_t$ and, consequently, of the score target
in Eq.~\eqref{eq:score-loss-main}. In particular, it shows that $h_t(r)$ is
the posterior expected corruption precision,
$\mathbb E[W_t^{-1}\mid \lVert u\rVert=r]$, given the observed residual
magnitude. Gaussian scale mixtures are a classical latent-variable
construction for symmetric non-Gaussian and heavy-tailed distributions
\citep{andrews1974scale,west1987scale}. The identity that a marginal score
is obtained by posterior averaging of latent component scores is a standard
latent-variable score identity \citep{louis1982observed}; in score-based
terms, it is closely connected to the conditional-score identities
underlying denoising score matching \citep{vincent2011dsm} and modern
score-based diffusion/SDE formulations \citep{song2021sde}. The
posterior-precision interpretation is also related in spirit to
empirical-Bayes score identities such as Tweedie's formula
\citep{efron2011tweedie}; here, however, the latent quantity being inferred
is the corruption precision $W_t^{-1}$ rather than an unknown clean mean.
The result is not specific to one non-Gaussian corruption law: it holds for
the whole isotropic Gaussian scale-mixture family.
Proposition~\ref{prop:score-uncertainty} gives a complementary uncertainty
interpretation of the same latent-scale posterior, after which the
corollaries separate the family-level cases that are relevant to the
corruptions used in StochLOB.

\begin{proposition}[Posterior uncertainty of the latent component score]
\label{prop:score-uncertainty}
Under the conditions of Proposition~\ref{prop:ht-posterior-precision}, define
the score of the Gaussian component associated with the latent scale $W_t$ as
\[
    S_t(u,W_t)
    :=\nabla_u\log\phi_{W_t}(u)
    =-W_t^{-1}u.
\]
Then the marginal score is the posterior mean of this latent component score,
\[
    \mathbb E\!\left[S_t(u,W_t)\mid u\right]
    =\nabla_u\log q_t(u)
    =-u\,h_t(r),
    \qquad r=\lVert u\rVert,
\]
and its posterior covariance is
\[
    \operatorname{Cov}\!\left(S_t(u,W_t)\mid u\right)
    =\operatorname{Var}\!\left(
        W_t^{-1}\mid \lVert u\rVert=r
    \right)uu^\top.
\]
Consequently, define the scalar posterior score uncertainty by
\[
    \mathcal U_t(r)
    :=\operatorname{tr}\!\left[
        \operatorname{Cov}\!\left(S_t(u,W_t)\mid u\right)
    \right].
\]
Then
\[
    \mathcal U_t(r)
    =r^2\operatorname{Var}\!\left(
        W_t^{-1}\mid \lVert u\rVert=r
    \right)
    =-r\,h_t'(r)
    \ge 0.
\]
Thus the same posterior uncertainty in corruption precision that makes $h_t$
residual-adaptive also measures uncertainty about the component-level
restoring score.
\end{proposition}

\begin{proof}
By the component-score identity above,
\[
    S_t(u,W_t)=-W_t^{-1}u.
\]
Conditioning on the observed residual $u$ and applying
Proposition~\ref{prop:ht-posterior-precision} gives
\[
    \mathbb E\!\left[S_t(u,W_t)\mid u\right]
    =-u\,\mathbb E\!\left[W_t^{-1}\mid u\right]
    =-u\,h_t(r)
    =\nabla_u\log q_t(u).
\]
Since $u$ is fixed under this conditional distribution,
\begin{align*}
    \operatorname{Cov}\!\left(S_t(u,W_t)\mid u\right)
    &=\mathbb E\!\left[
        \bigl(W_t^{-1}-h_t(r)\bigr)^2\mid u
      \right]uu^\top\\
    &=\operatorname{Var}\!\left(W_t^{-1}\mid u\right)uu^\top\\
    &=\operatorname{Var}\!\left(
        W_t^{-1}\mid \lVert u\rVert=r
      \right)uu^\top,
\end{align*}
where the last equality again follows from isotropy. Taking the trace and
using
\[
    \operatorname{tr}(uu^\top)
    =\lVert u\rVert^2
    =r^2
\]
gives
\[
    \mathcal U_t(r)
    =r^2\operatorname{Var}\!\left(
        W_t^{-1}\mid \lVert u\rVert=r
    \right).
\]
Finally, Proposition~\ref{prop:ht-posterior-precision} gives
\[
    h_t'(r)
    =-r\,\operatorname{Var}\!\left(
        W_t^{-1}\mid \lVert u\rVert=r
    \right),
\]
so
\[
    \mathcal U_t(r)=-r\,h_t'(r)\ge 0.
\]
This is posterior uncertainty induced by the latent corruption scale $W_t$;
it does not represent parameter uncertainty in the learned score network.
\end{proof}

\begin{corollary}[Deterministic-scale corruption]
\label{cor:ht-deterministic}
Suppose $W_t=w_t$ deterministically for some $w_t>0$. Then
\[
    h_t(r)=\frac{1}{w_t},
    \qquad
    h_t'(r)=0,
\]
and
\[
    \nabla_u\log q_t(u)=-\frac{u}{w_t}.
\]
Thus deterministic-scale Gaussian corruption uses a fixed restoring
precision, independent of the residual magnitude. StochLOB (Gaussian) is
the particular case $w_t=\sigma_t^2$.
\end{corollary}

\begin{proof}
If $W_t=w_t$ almost surely, then $W_t^{-1}=w_t^{-1}$ almost surely.
Proposition~\ref{prop:ht-posterior-precision} therefore gives
\[
    h_t(r)
    =
    \mathbb E\!\left[W_t^{-1}\mid \lVert u\rVert=r\right]
    =
    w_t^{-1},
\]
so $h_t'(r)=0$. Substituting into Eq.~\eqref{eq:gsm-score-main} gives the score expression.
\end{proof}

\begin{corollary}[Non-degenerate random-scale corruption]
\label{cor:ht-random-scale}
Suppose the scale distribution of $W_t$ is non-degenerate. Then, for every
$r>0$,
\[
    h_t'(r)<0.
\]
Hence every genuinely random isotropic Gaussian scale mixture has a
strictly residual-adaptive posterior precision: as the residual magnitude
increases, the per-unit restoring precision decreases. The jump-Poisson and
alpha-stable scale-mixture corruptions used in StochLOB are examples of this
random-scale family.
\end{corollary}

\begin{proof}
By Proposition~\ref{prop:ht-posterior-precision},
\[
    h_t'(r)
    =
    -r\,\operatorname{Var}\!\left(
        W_t^{-1}\mid \lVert u\rVert=r
    \right).
\]
For a non-degenerate scale distribution, the posterior distribution of
$W_t^{-1}$ remains non-degenerate for every finite residual because the
Gaussian likelihood is strictly positive for every $W_t>0$. Hence the
posterior variance is strictly positive, and $h_t'(r)<0$ for $r>0$.
\end{proof}

The resulting reduction in effective precision for larger residuals is
consistent with the classical robustness interpretation of heavy-tailed
scale-mixture models, such as Student-$t$ error models, in which extreme
residuals exert less influence than under Gaussian errors
\citep{lange1989robust}. Corollary~\ref{cor:ht-random-scale} strengthens this
intuition for the present setting by establishing strict residual adaptation
for every non-degenerate isotropic Gaussian scale mixture satisfying the
proposition's regularity conditions.

\begin{corollary}[Bounded random-scale corruption]
\label{cor:ht-bounded-scale}
Suppose the scale distribution is non-degenerate and satisfies
\[
    0<W_t\le W_{t,\max}
\]
almost surely for some finite $W_{t,\max}$. Then
\[
    h_t(r)\ge \frac{1}{W_{t,\max}},
\]
and, for $r>0$,
\[
    h_t'(r)<0.
\]
Thus bounded random-scale corruption retains residual adaptation but has a
positive precision floor. If the scale distribution also places positive
probability at $W_{t,\max}$, then
\[
    h_t(r)
    \longrightarrow
    \frac{1}{W_{t,\max}}
    \qquad
    \text{as }r\to\infty.
\]
\end{corollary}

\begin{proof}
Since $W_t\le W_{t,\max}$, taking reciprocals gives
\[
    W_t^{-1}\ge W_{t,\max}^{-1}.
\]
Therefore, by Proposition~\ref{prop:ht-posterior-precision},
\[
    h_t(r)
    =
    \mathbb E\!\left[
        W_t^{-1}\mid \lVert u\rVert=r
    \right]
    \ge
    \frac{1}{W_{t,\max}}.
\]
Strict decrease follows from Corollary~\ref{cor:ht-random-scale}.

For the limiting statement, assume the scale distribution places positive
probability at $W_{t,\max}$. For any smaller scale $w<W_{t,\max}$, the
Gaussian likelihood ratio relative to the maximal-scale component is
\[
    \frac{
        \mathcal N(u;0,wI_d)
    }{
        \mathcal N(u;0,W_{t,\max}I_d)
    }
    =
    \left(
        \frac{W_{t,\max}}{w}
    \right)^{d/2}
    \exp\!\left[
        -\frac{r^2}{2}
        \left(
            \frac{1}{w}
            -
            \frac{1}{W_{t,\max}}
        \right)
    \right].
\]
Since $w<W_{t,\max}$, this ratio tends to zero as $r\to\infty$.
Consequently, increasingly large residuals are dominated by the
maximal-scale component, and the posterior expected precision converges to
$W_{t,\max}^{-1}$.
\end{proof}

\begin{corollary}[Polynomially heavy-tailed radial corruption]
\label{cor:ht-polynomial-tail}
Suppose the isotropic marginal corruption density can be written as
\[
    q_t(u)=g_t(r),
    \qquad
    r=\lVert u\rVert,
\]
and, for some $m>d$,
\[
    g_t(r)=L_t(r)r^{-m}
\]
for all sufficiently large $r$, where $L_t(r)>0$ is differentiable and
satisfies
\[
    \frac{rL_t'(r)}{L_t(r)}\longrightarrow 0.
\]
Then
\[
    h_t(r)\sim\frac{m}{r^2},
\]
and consequently
\[
    \left\lVert\nabla_u\log q_t(u)\right\rVert
    =r\,h_t(r)
    \sim\frac{m}{r}
    \longrightarrow 0.
\]
Thus polynomially heavy-tailed radial corruptions have the stronger
far-tail property that the absolute restoring score eventually decreases
toward zero.
\end{corollary}

\begin{proof}
For an isotropic density,
\[
    \nabla_u\log q_t(u)
    =
    \frac{g_t'(r)}{g_t(r)}\frac{u}{r}.
\]
Comparing with Eq.~\eqref{eq:gsm-score-main} gives
\[
    h_t(r)
    =
    -\frac{g_t'(r)}{r\,g_t(r)}.
\]
Since $g_t(r)=L_t(r)r^{-m}$,
\[
    \frac{g_t'(r)}{g_t(r)}
    =
    \frac{L_t'(r)}{L_t(r)}
    -
    \frac{m}{r}
    =
    -\frac{m+o(1)}{r}.
\]
Therefore,
\[
    h_t(r)\sim\frac{m}{r^2},
\]
and multiplying by $r$ gives the stated score-magnitude asymptotic.
\end{proof}

The family-level results above include the corruption choices used in
StochLOB as particular examples. For StochLOB (Gaussian),
$W_t=\sigma_t^2$ deterministically, so
Corollary~\ref{cor:ht-deterministic} gives the fixed precision
$h_t(r)=\sigma_t^{-2}$.

For StochLOB (jump-Poisson),
\[
    W_t
    =
    \sigma_t^2+\sum_{j=1}^{N_t}S_j,
    \qquad
    N_t\sim\operatorname{Poisson}(\Lambda_t),
    \qquad
    S_j\sim\operatorname{Gamma}(\kappa,\theta).
\]
For a nontrivial jump step $\Lambda_t>0$, this scale distribution is
non-degenerate, so Corollary~\ref{cor:ht-random-scale} gives
$h_t'(r)<0$ for $r>0$. In addition, $W_t\ge\sigma_t^2$ almost surely,
so
\[
    0<h_t(r)<\frac{1}{\sigma_t^2}
\]
for every finite residual. Thus the jump-Poisson corruption is one concrete
random-scale construction in which the effective posterior precision falls
below the fixed Gaussian value as the residual grows.

For the ideal, unclipped isotropic $\alpha$-stable corruption, the
connection between stable laws and Gaussian scale mixtures is classical
\citep{west1987scale}, while the standard far-tail density behavior
\citep{samorodnitsky1994stable} is
\[
    q_t(u)
    \sim
    C_t\lVert u\rVert^{-(d+\alpha)},
    \qquad
    \lVert u\rVert\to\infty.
\]
The corresponding smooth density asymptotics for multidimensional
\(\alpha\)-stable laws \citep{zaigraev2006asymptotic} give the derivative
behavior required by Corollary~\ref{cor:ht-polynomial-tail}, with
\(m=d+\alpha\). Hence
\[
    h_t(r)
    \sim
    \frac{d+\alpha}{r^2},
\]
and
\[
    \left\lVert\nabla_u\log q_t(u)\right\rVert
    =
    r\,h_t(r)
    \sim
    \frac{d+\alpha}{r}
    \longrightarrow 0.
\]
Thus the ideal alpha-stable construction is an example of the stronger
polynomial-tail geometry in which both the per-unit restoring precision and
the absolute restoring score weaken in the far tail.

For the clipped alpha-stable corruption implemented in StochLOB,
\[
    A_{\mathrm{clip}}
    =
    \min\!\left(A,Q_{0.999}(A)\right),
    \qquad
    W_t=\sigma_t^2A_{\mathrm{clip}},
\]
and
\[
    W_{t,\max}
    :=
    \sigma_t^2Q_{0.999}(A).
\]
The resulting scale is non-degenerate, bounded above by $W_{t,\max}$, and
places positive probability at the maximal scale because of clipping.
Corollary~\ref{cor:ht-bounded-scale} therefore gives
\[
    h_t'(r)<0,
    \qquad
    h_t(r)\ge\frac{1}{W_{t,\max}},
\]
with
\[
    h_t(r)
    \longrightarrow
    \frac{1}{W_{t,\max}}
    \qquad
    \text{as }r\to\infty.
\]
Thus clipping preserves residual-adaptive downweighting but imposes a
positive precision floor in the extreme tail.

These statements provide a theoretical mechanism that complements the
paper's empirical stress results rather than replacing them. They do not
imply that every non-Gaussian corruption must improve macro-F1. Instead, they
show that the proposed non-Gaussian objectives possess a family-level
residual-adaptation mechanism that the Gaussian baseline cannot possess.
This interpretation is complementary to recent work that replaces Gaussian
perturbations by heavy-tailed or L\'evy alternatives in denoising score
matching and diffusion models
\citep{deasy2021heavytailed,yoon2023levy,pandey2025heavytailed,shariatian2025dlpm}:
rather than proposing a new generative reverse process, the analysis here
isolates the posterior-precision mechanism induced by the latent scale in the
supervised StochLOB objective. Proposition~\ref{prop:score-uncertainty}
further shows that, when the corruption scale is random, the same latent-scale
posterior also induces uncertainty about the component-level restoring score,
with this posterior score uncertainty determined by the variance of the
latent corruption precision. Thus the non-Gaussian scale-mixture objectives
differ from the Gaussian baseline not only in how their effective restoring
precision adapts to the observed residual, but also in allowing uncertainty
over the latent corruption scale underlying that restoration. The stress
experiments determine whether these structural properties are useful for the
observed LOB prediction problem.

\subsection{Mathematical Rationale for Future Work}
\label{app:future-work-math}

The future directions in Section~\ref{sec:discussion} follow from two related
facts. First, the current score target is a posterior precision. Second,
robustness transfer from a training corruption law to an unseen stress law is
controlled by both the size and the direction of the shocks covered by the
training law.

\paragraph{From scalar precision to stress coverage.}
At a fixed noise scale, the present model uses an isotropic scale mixture
\[
    q_t(u)=\int \mathcal N(u;0,wI_d)\,p_t(dw),
    \qquad
    \nabla_u\log q_t(u)
    =
    -u\,\E[W_t^{-1}\mid u].
\]
This identifies the quantity that future work should generalize: the scalar
posterior precision \(h_t(\lVert u\rVert)\). The next question is whether a
corruption law \(q_t\) that is useful during training also covers a deployment
stress law \(\nu\). Let \(L(u)\ge0\) be the loss incurred under a perturbation
\(u\), and assume the perturbation is a nuisance for the label,
\[
    p(y\mid x_0+u)=p(y\mid x_0).
\]
If \(\nu\ll q_t\), Cauchy--Schwarz gives the second-order transfer bound,
familiar from importance weighting \citep{cortes2010importance}:
\[
    \E_{\nu}[L]
    =
    \int L(u)\frac{\nu(u)}{q_t(u)}q_t(u)\,du
    \le
    \sqrt{\E_{q_t}[L^2]}\,
    \sqrt{1+\chi^2(\nu\Vert q_t)},
    \qquad
    1+\chi^2(\nu\Vert q_t)
    =
    \int \frac{\nu(u)^2}{q_t(u)}\,du .
\]
Thus the training objective can transfer to stress only when two quantities
are controlled: a training-side second moment and the coverage constant
\(1+\chi^2(\nu\Vert q_t)\). This is a sufficient condition, not a claim that
large \(\chi^2\) necessarily implies poor accuracy. The same constant governs
the sample size required for importance sampling, where it is known to grow
with an intrinsic dimension of the mismatch \citep{agapiou2017importance}.

The fixed-scale statement matches the implementation because noise levels are
sampled from a finite grid. If a scale \(t\) is sampled with probability
\(\pi_t\) and loss weight \(w_t\), a weighted average objective controls the
per-scale loss up to the factor \((\pi_t w_t)^{-1}\). For a continuous
schedule, an analogous deployed statement would require a finite average or
supremum of \(1+\chi^2(\nu\Vert q_t)\) over \(t\).

\paragraph{Radial--angular factorization.}
Write \(u=r\theta\), where \(r>0\) and
\(\theta\in\mathbb S^{d-1}\). Suppose the stress law and corruption law have
regularly varying tails
\[
    \nu(r\theta)
    =
    r^{-(d+\alpha_\nu)}L_\nu(r)\psi_\nu(\theta)(1+o(1)),
    \qquad
    q_t(r\theta)
    =
    r^{-(d+\alpha_q)}L_q(r)\psi_q(\theta)(1+o(1)),
\]
uniformly in \(\theta\), following the standard radial--angular form of
multivariate regular variation
\citep{bingham1987regular,resnick2007heavy,basrak2002characterization}. Here
\(\alpha_\nu\) and \(\alpha_q\) are tail indices, \(L_\nu\) and \(L_q\) are
slowly varying radial factors, and \(\psi_\nu,\psi_q\) are angular densities
with respect to the uniform probability measure on the sphere. Splitting the
integral at \(r_0\) gives
\[
    1+\chi^2(\nu\Vert q_t)-C_0(r_0)
    =
    (1+o(1))\,
    I_{\mathrm{rad}}(r_0)\,
    I_{\mathrm{ang}},
\]
where
\[
\begin{aligned}
I_{\mathrm{rad}}(r_0)
&=
|\mathbb S^{d-1}|
\int_{r_0}^{\infty}
r^{-1-(2\alpha_\nu-\alpha_q)}
\frac{L_\nu(r)^2}{L_q(r)}\,dr,
\\[4pt]
I_{\mathrm{ang}}
&=
\int_{\mathbb S^{d-1}}
\frac{\psi_\nu(\theta)^2}{\psi_q(\theta)}
\,d\sigma(\theta)
\\
&=
\exp\!\left(
D_2(\psi_\nu\Vert\psi_q)
\right).
\end{aligned}
\]
The angular term is the exponentiated order-2 R\'enyi divergence
\citep{renyi1961entropy,vanerven2014renyi}.
The radial factor is finite when \(\alpha_q<2\alpha_\nu\) and divergent
when \(\alpha_q>2\alpha_\nu\); the critical case depends on the slowly
varying term. This explains why Gaussian corruption is a weak stress proxy
for heavy-tailed LOB residuals: a Gaussian \(q_t\) decays exponentially, while
a heavy-tailed \(\nu\) decays polynomially, so \(\nu^2/q_t\) is not
integrable in the far tail.

The factorization also shows a limitation of the current isotropic design.
Heavy tails improve radial coverage, but they do not automatically cover the
directions in which market stress occurs. If the corruption allocates mass
uniformly over directions while realized stress lies near a low-dimensional
signed order-flow subspace, the angular term \(I_{\mathrm{ang}}\) can make the
bound numerically vacuous even when the radial tail index is adequate.

\paragraph{High-dimensional angular penalty.}
The scale of this effect is large for LOB windows. Let \(\nu\) be an
elliptical stress law with shape matrix \(\Sigma_\nu\), eigenvalues
\(\lambda_1,\ldots,\lambda_d\) normalized to unit mean, and tail index
\(\alpha_\nu\) \citep{cambanis1981elliptical,hult2002multivariate}.
Against isotropic corruption, define
\[
    Q(\theta)=\theta^\top\Sigma_\nu^{-1}\theta,
    \qquad
    p_\nu=\frac{d+\alpha_\nu}{2},
    \qquad
    K(s)=\log \E_{\operatorname{Unif}}\!\left[Q(\theta)^{-s}\right].
\]
The angular divergence has the exact representation
\[
    D_2(\psi_\nu\Vert\operatorname{Unif})
    =
    K(2p_\nu)-2K(p_\nu).
\]
For a nearly spherical spectrum, with
\(\operatorname{CV}(\lambda)=
\operatorname{std}(\lambda)/\operatorname{mean}(\lambda)\), the cumulant
expansion gives
\[
    D_2(\psi_\nu\Vert\operatorname{Unif})
    =
    \frac{(d+\alpha_\nu)^2}{2(d+2)}
    \operatorname{CV}^2(\lambda)\,(1+o(1))
    =
    \frac{d}{2}\operatorname{CV}^2(\lambda)(1+O(d^{-1})).
\]
The coefficient is proportional to dimension. For the \(60\times31\) crypto
LOB window used in this paper, \(d=1860\), so keeping this angular exponent
near order one requires approximately
\(\operatorname{CV}(\lambda)\lesssim d^{-1/2}\approx2.3\%\). A 20\% spectral
dispersion can multiply \(1+\chi^2(\nu\Vert q_t)\) by about \(10^{16}\), which
inflates the square-root transfer bound by about \(10^8\).

\begin{table}[t]
\centering
\begin{tabular}{@{}cccccc@{}}
\toprule
 & \multicolumn{2}{c}{symmetric spectrum} &
 \multicolumn{2}{c}{skewed spectrum} & \\
\cmidrule(lr){2-3}\cmidrule(lr){4-5}
\(\operatorname{CV}(\lambda)\) &
\(D_2\) & \(\exp(D_2)\) &
\(D_2\) & \(\exp(D_2)\) &
\(\exp(D_2/2)\) (sym.) \\
\midrule
0.01 & 0.09 & \(1.1\) & 0.09 & \(1.1\) & \(1.05\) \\
0.05 & 2.32 & \(1.0\times10^{1}\) & 2.37 & \(1.1\times10^{1}\) & \(3.2\) \\
0.10 & 9.26 & \(1.1\times10^{4}\) & 10.0 & \(2.2\times10^{4}\) & \(1.0\times10^{2}\) \\
0.20 & 36.6 & \(7.9\times10^{15}\) & 47.1 & \(2.9\times10^{20}\) & \(8.9\times10^{7}\) \\
\bottomrule
\end{tabular}
\caption{Angular penalty at \(d=1860\) and \(\alpha_\nu=1.62\). The
symmetric spectrum places half the coordinates at \(1+\operatorname{CV}\) and
half at \(1-\operatorname{CV}\); the skewed spectrum places the excess on
5\% of the coordinates. The angular factor multiplies the coverage constant,
and its square root multiplies the transfer bound.}
\label{tab:angular-penalty}
\end{table}

Table~\ref{tab:angular-penalty} makes the design implication explicit:
future corruption laws should not only be heavy-tailed; they should also
estimate the spectral geometry of realized stress, for which nonparametric
angular-measure estimators exist \citep{einmahl2001spectral}.

\paragraph{State-dependent and level-correlated shocks.}
Let \(z\) denote the observed market state before corruption; queue-reactive
models show that order-flow intensities depend strongly on this state
\citep{huang2015queue}. Replace the scalar variance \(W_tI_d\) by a
positive-definite random covariance \(\Sigma_t\). Student-\(t\) and
\(\alpha\)-stable diffusion kernels are special cases of this scale-mixture
construction \citep{pandey2025heavytailed,shariatian2025dlpm}:
\[
    q_t(u\mid z)
    =
    \int
    \mathcal N(u;0,\Sigma)\,p_t(d\Sigma\mid z).
\]
The score becomes
\[
    \nabla_u\log q_t(u\mid z)
    =
    -
    \E\!\left[\Sigma_t^{-1}\mid u,z\right]u.
\]
Thus \(h_t(r)I_d\) is replaced by the posterior precision matrix
\(\E[\Sigma_t^{-1}\mid u,z]\). This extension keeps the same posterior-score
principle, but lets the corruption focus probability mass on the level pairs,
assets, and temporal coordinates where stress is empirically concentrated.

\paragraph{Signed empirical jump replay.}
Real order-book jumps are not only large; they are signed and structured.
This can be represented by a random jump vector \(J_t\) and covariance
\(\Sigma_t\):
\[
    q_t(u\mid z)
    =
    \iint
    \mathcal N(u-j,\Sigma)\,p_t(dj,d\Sigma\mid z).
\]
Bayes' rule gives the score target
\[
    \nabla_u\log q_t(u\mid z)
    =
    -
    \E\!\left[
        \Sigma_t^{-1}(u-J_t)
        \mid u,z
    \right].
\]
The current model is the special case \(J_t=0\) and \(\Sigma_t=W_tI_d\). An
empirical jump-replay extension would estimate \(p_t(dj,d\Sigma\mid z)\) from
observed order-book shock residuals, so the auxiliary task would learn not
only how much to denoise a large residual, but also which signed shock pattern
made that residual plausible. In the coverage notation above, this is a way
to reduce both the radial mismatch and the angular R\'enyi penalty.

\paragraph{Ablations and seed-level uncertainty.}
The training objective can be written as
\[
    \mathcal L(\theta;\lambda_s)
    =
    \mathcal L_{\CE}(\theta)
    +
    \lambda_s\mathcal L_{\mathrm{score}}(\theta),
\]
where \(\lambda_s\) controls score matching. For any evaluation metric \(M\),
the marginal effect of the score loss can be measured by
\[
    \Delta_s
    =
    M(\lambda_s)-M(0).
\]
Across \(m\) random seeds, the uncertainty of either effect can be reported as
\[
    \bar{\Delta}
    \pm
    t_{m-1,0.975}\frac{\widehat{\sigma}_{\Delta}}{\sqrt m}.
\]
This is the direct way to quantify the contribution of the score target. The
same design can compare temporal readouts by treating the readout family as a
controlled factor and reporting paired seed-level differences in \(M\).

\paragraph{Cost-aware portfolio thresholds.}
Let \(\pi_{\tau,t}\) be the portfolio generated from model probabilities using
threshold vector \(\tau\), let \(r_{t+1}\) be next-period asset returns, and
let \(c_t(\Delta\pi_{\tau,t})\) denote transaction costs. A validation-tuned
threshold should solve
\[
    \tau^\star
    =
    \arg\max_{\tau}
    \frac{1}{T}
    \sum_{t=1}^{T}
    \left[
        \pi_{\tau,t}^{\top}r_{t+1}
        -
        c_t(\Delta\pi_{\tau,t})
        -
        \frac{\gamma}{2}
        \pi_{\tau,t}^{\top}\Omega_t\pi_{\tau,t}
    \right],
\]
This is the single-period mean--variance objective with transaction costs
studied by \citet{garleanu2013dynamic} and
\citet{boyd2017multiperiod}; the threshold vector \(\tau\) replaces their
continuous trade variable. Here \(\Omega_t\) is a risk estimate and
\(\gamma\ge0\) controls risk
aversion. This makes the Feishu extension explicit: the classifier supplies
probabilities, but the final trading rule should be selected by expected
net utility under market frictions, not by fixed cutoffs chosen independently
of costs.

\end{document}